\documentclass[openacc]{rsproca_new}
\usepackage{amsmath}
\usepackage{amssymb}
\usepackage{amsthm}
\usepackage{tikz}
\usepackage{pgfplots}
\usepackage{ifthen}
\usepackage{xfp, siunitx}
\usepackage{subcaption}
\usepackage{import}
\usepackage{booktabs}
\usepackage{pgf}
\pgfplotsset{
    tick label style={font=\small},
    label style={font=\small},
    legend style={font=\footnotesize},
    compat=1.18
}

\everymath=\expandafter{\the\everymath\displaystyle}

\makeatletter\@ifpackageloaded{underscore}{}{\usepackage[strings]{underscore}}\makeatother

\usepackage{listings}
\usetikzlibrary{shapes, arrows.meta, positioning, calc, patterns, colorbrewer}

\newcommand{\CNT}[1]{\textsf{CNT}_{#1}}
\newcommand{\CF}{\textsf{CF}}
\newcommand{\SF}{\textsf{SF}}

\newtheorem{Definition}{\bf Definition}[section]
\newtheorem{Proposition}{\bf Proposition}[section]

\titlehead{Research}

\begin{document}

\title{Quantum-Like Contextuality in Large Language Models}

\author{Kin Ian Lo $^{1}$, Mehrnoosh Sadrzadeh $^{1}$ and Shane Mansfield $^{2}$}

\address{%
$^{1}$ \quad University College London; \{kin.lo.20,m.sadrzadeh\}@ucl.ac.uk\\
$^{2}$ \quad Quandela; shane.mansfield@quandela.com}

\subject{artificial Intelligence, quantum Computing}


\corres{Mehrnoosh Sadrzadeh\\
\email{m.sadrzadeh@ucl.ac.uk}}

\begin{abstract}
Contextuality is a distinguishing feature of quantum mechanics and there is growing evidence that it is a necessary condition for quantum advantage. 
In order to make use of it, researchers have been asking whether similar phenomena arise in other domains. The answer has been yes, e.g. in behavioural sciences. However, one has to move to frameworks that take some degree of signalling into account. 
Two such frameworks exist: (1) a signalling-corrected sheaf theoretic model, and (2) the Contextuality-by-Default (CbD) framework.
This paper provides the first large scale experimental evidence for a yes answer in natural language. 
We construct a linguistic schema modelled over a contextual quantum scenario, instantiate it in the Simple English Wikipedia and extract probability distributions for the instances using the large language model BERT. 
This led to the discovery of 77,118 sheaf-contextual and 36,938,948 CbD contextual instances. 
We proved that the contextual instances came from semantically similar words, by deriving an equation between degrees of contextuality and Euclidean distances of BERT's embedding vectors. 
A regression model further reveals that Euclidean distance is indeed the best statistical predictor of contextuality. Our linguistic schema is a variant of the co-reference resolution  challenge. These results are an indication that  quantum methods may be advantageous in language tasks.  
\end{abstract}

\maketitle

\section{Introduction}


Contextuality is a statistical phenomenon that only occurs in quantum mechanics but not in classical physics.
Intuitive assumptions about the world, such as the assumption that the outcome of a measurement does not depend what other measurements are performed at the same time, are violated in quantum mechanics.
The study of measurement scenarios arising, e.g., from the EPR paradox~\cite{EPR} and Bell's theorem~\cite{Bell1966,Bell1964}, led quantum mechanics to be the first field of science to formally deal with the notion of contextuality.
Scientists argued that quantum theory should be contextual in order to be sound and different mathematical formalisms were introduced to analyse this. 
Quantum-like contextuality turned out to be, essentially, the impossibility of having a global explanation to a set of local observations on a system. 
Such impossibility only occurs in systems with incompatible observables, i.e.\ a simultaneous global observation of all observables are not possible.
Over the last number of years it has been proved that contextuality is a resource for lifting a linear computer to a universal classical computer~\cite{Anders2009, Raussendorf2013,
Mansfield2018} and that contextuality is necessary for successful magic state distillation
\cite{Howard2014}, a key component in fault-tolerant quantum computing schemes.
Based on these recent results, we can speculate that contextuality is, in a rough sense, a resource that possess extra computational power that is absent in non-contextual systems.
A natural conjecture is that simulating contextual systems is harder than simulating non-contextual systems and that quantum computers, which are contextual, are better at simulating contextual systems than classical computers.

There exist a number of different frameworks for the formal treatment of contextuality. 
The sheaf-theoretic framework of \cite{Abramsky2011} is amongst the ones that connects the statistical data collected from quantum experiments to the laws of physics.
One of these laws is the \emph{no-signalling} property. The sheaf-theoretic framework can only formalise contextual scenarios that are non-signalling. 
However, the examples we are aiming to study are highly unlikely to be non-signalling. We remedy this by using a recent extension of the framework that is designed to model experiments in the presence of experimental noise~\cite{Emeriau2022}. In this extension,  the authors derive an inequality to check the contextuality of systems of measurements with noisy data.  As a result, this extensions allows for presence of  some degree of signalling. We use this extension and its newly derived inequality to check the contextuality of our examples. 

Sheaves are well used structures for analysing and fusing data; they have been applied to signal processing \cite{robinson2017}, vision and gene expression \cite{Kvinge2021}, and social robotics \cite{murimi2021}; for a detailed study see \cite{curry2014sheaves,Hansen2019}. 
They have been used to detect logical inconsistencies in large language models \cite{huntsman2024prospects} and knowledge graph embeddings \cite{Gebhart2023}.
They have also been used to formalise the geometry of learning in graph neural networks \cite{Bodnar2023}. Away from deep learning, sheaves have been used to formalise the semantic consistency of anaphora resolution  
\cite{AbramskySadr2014}, to unify different aspects of semantic composition \cite{Philips2019}, and to model consistent text \cite{bradley2021}.

Quantum-like contextuality in the presence of signalling has also  been studied in a framework known as \emph{Contextuality-by-Default} (CbD). This framework was used to  probe contextuality  in behavioural sciences~\cite{Dzhafarov2015} and more recently also in natural language \cite{Wang2021a,Wang2021b}. In natural language, it was shown that  pairs of semantically ambiguous words can produce contextual systems that resemble the Bell-CHSH quantum measurement scenario.  Semantic ambiguity is not the only scenario that can produce contextual systems. In previous work~\cite{Lo2022,Lo2023}, we proposed a  linguistic construction that exhibits the contextuality of anaphoric ambiguities coming from a set of sentences with ambiguous pronouns in them. Neither of these previous works were large scale. The work done in semantic ambiguity, considered a dataset consisting of a few dozen of hand crafted 2-word phrases. The joint  probability distributions of the words within these phrases were also constructed manually, by eyeing the contexts surrounding them in a sample projection of a small corpus.  The dataset of the previous work on anaphoric ambiguities only had 11 pairs of nouns in it. Although the probability distributions were mined from the large language model BERT, the pairs themselves were chosen manually. As a result, these papers  offered anecdotal evidence on the quantum-like contextuality of natural language and only the latter of them indicated the presence of a connection with large language models. 

The purpose behind this paper is to offer  systematic large scale evidence for the presence of quantum-like contextuality in natural language. We work with a schema modelled over the minimal quantum contextual scenario. The schema has two nouns, three adjectives and two pronouns.  Each pronoun is modified by a pair of adjectives, making it refer meaningfully  to one of the two nouns.  We go through all the noun pairs that that are modified by the same triples of adjectives in the simple English Wikipedia corpus \cite{wikidump} and automatically  instantiating the schema using this data.   We then use  BERT to extract probability distributions for co-occurrences of the noun pairs in the context of the pronouns and adjectives. This provided us with 51,966,480 instances of the schema, of which 77,118  were contextual in the sheaf-theoretic  and 36,938,948 in the CbD framework.

The systematic and automatic nature of our data collection methodology led to a plenitude of examples, many more than the ones discovered in previous work, by  Wang et al.\!~\cite{Wang2021b},  and  by ourselves~\cite{Lo2022,Lo2023}.  This enabled  us to study what makes an instance of the schema more likely to be contextual. We took two steps to further this study. First, we derived an equation between a parameter relating to contextuality and the Euclidean distance between the BERT prediction vectors of the noun pairs used in an instance of the schema.  Next, we  considered a set of features of the probability distributions provided by BERT, including Euclidean distance, as well as the information theoretic content of the instances as computed by entropy and trained a polynomial regression model over them. Amongst these features,  Euclidean distances demonstrated statistically significant correlations  with both notions of sheaf theoretic and CbD contextuality. Euclidean distance is closely related to  semantic similarity, which depends on co-occurrence in the same linguistic context. In order to confirm its correlation with quantum-like contextuality, we restricted the noun pairs of the dataset to those that had a high semantic similarity. This increased the percentage of the contextual instances, as measured in this controlled subset of the dataset versus in the whole dataset. This showed that the number of contextual examples rises as we increase the threshold over which semantic similarity is calculated.   These results demonstrate that finding a correlation between context-sensitive features of  large language models and quantum contextuality is beyond co-incidence.  Our main future aim is to use this overlap to investigate how  quantum-like advantage may  result  from this situation and how it can lead to more efficient large language algorithms.


\section{Background on Ambiguities in Natural Language}

Natural language is ambiguous at the different levels of syntax, semantics and pragmatics. Semantic ambiguities result from the fact that words  have different meanings.   For example, ``bat'' has an animal meaning and a sports meaning,  ``plant'' can mean a living organism such as a tree or a shrub, or a manufacturing industrial unit, such as a power plant. 
\emph{Word Sense Disambiguation} is a long standing task and evaluation method in Natural Language Processing.  Here the goal is to identify which meaning  of a word is being used in a context. Syntactic ambiguities arise from the fact that words can have different grammatical roles, e.g.  ``plant'' can be a noun or a verb. They also arise in situations where a grammatical compound attaches itself to other compounds using different roles. An example here is the preposition phrase ``in pyjamas''  in the sentence ``I saw a man in pyjamas'', where it is not clear whether it is modifying the subject ``I'' or the object ``a man''. 
Another set of  major ambiguities arise in pragmatics, where a key challenge is the \emph{coreference resolution} task. This type of ambiguity happens   when one has to decide which part of the discourse  is being referred by another expression in its context. An example of this type of ambiguity is visible on the discourse ``I have a Toyota RAV4 and a Toyota Aqua, it is the hybrid one for which I need an MOT today.'', where the pronoun ``it'' can in principle refer to either car. Resolving these ambiguities is an important part of language engines such as dialogue systems and  question answering. For instance, in an automatic MOT booking  system, the interfacing language engine should some how (in this case using external knowledge or a database of facts) deduce  that the``it''  in the  previous example sentence is referring to ``Toyota Aqua'' and not ``Toyota RAV4''.

Different  coreference resolution algorithms focus on different classes of referring expressions. Pronouns are in the class of definite referents and refer to entities that are identifiable from the context, because they have been mentioned before (or after). In the  discourse `Dawn called the AA. The car had broken down and she had no choice',  the pronoun `she' refers to the definite noun phrase `Dawn' and is an instance of the linguistic phenomena  \emph{anaphora}. Despite presence of linguistics agreement in the anaphoric relations, such as gender and number, ambiguities are common. In ``Dawn texted Wendy. Her car had broken down.'',  or ``Dawn phoned Wendy. She was upset and needed help.'', it  is not clear whose car was broken or who was upset.   Pronominal coreference relations are many-to-many and the ambiguities arising from them take complex forms.  A  pronoun can refer to multiple referents and multiple pronouns can refer to the same referent. In the  discourse ``There is a man carrying a boy. He is tired and worn out. He is crying. He is quiet. They are moving fast.'', the first \emph{He} can refer to both \emph{man} or \emph{boy}, but the second \emph{He} most probably refers to \emph{a boy} and the third \emph{He} to \emph{a man}.

The different choices that give rise to ambiguities, be it in the choice of the meaning of a word in Word Sense Disambiguation task, or the choice of the potential grammatical structure of a sentence or the referent of an expression in coreference resolution, give rise to probability distributions. 
An ambiguous word can be treated as an observable which can have possible outcomes. In case of meaning ambiguities, these outcomes are possibilities over the semantic interpretations of the word.
A probability distribution over the outcomes can then be defined using the frequency of occurrences of the possible interpretations in a corpus, e.g. the  English Wikipedia corpus that we have used in this paper (although mining the plausibility judgements from human subjects is also possible, see previous work \cite{Wang2021a,Wang2021b,Lo2022,Lo2023}).  A single observable is not sufficient to support contextuality, instead pairs of ambiguous words are needed. A pair of words is thought of as a pair of compatible observables measured simultaneously. 

The work of \cite{Wang2021a,Wang2021b} focused on meaning ambiguities.  In this paper, we focus on coreference ambiguities and model contextual features of   ambiguities  arising from anaphoric  reference relations. An  identical approach can be taken if the relationship is cataphoric. We treat the pronouns as observables, of which the measurement outcomes are the possible referents of each pronoun. In what follows we describe the mathematical setting we used, detail how to use it to model ambiguous anaphoric references,  explain how we found contextual examples, and present some of the contextual examples.

\section{Background on the sheaf-theoretic framework for contextuality}
\label{sec:sheafframework}
In the sheaf-theoretic framework of contextuality \cite{Abramsky2011}, a measurement scenario is a tuple $\langle \mathcal{X}, \mathcal{M}, \mathcal{O} \rangle$ with the data: $\cal X$, a set of observables; $\cal M$, a measurement cover; and $\cal O$, a set of measurement outcomes. 
An observable in $\mathcal{X}$ is a quantity that can be measured to give one of the  outcomes in $ \mathcal{O} $. 
A subset of simultaneously measurable observables of $\mathcal{X}$ is called a measurement context (or simply context).  
The measurement cover $\mathcal{M}$ is a collection of contexts which covers $\mathcal{X}$, i.e.\ the union of all contexts in $\mathcal{M}$ is $ \mathcal{X}$.

For every measurement context, we can perform a number of repeated simultaneous measurements on the observables in the context. The gathered statistics can then be used to reconstruct an estimated joint probability distribution. 
One can also instead calculate the joint distribution exactly using an underlying theory of the concerned system, e.g.\ using Born's rule in quantum mechanics for a quantum system.

An \emph{empirical model} refers to a collection of such joint probability distributions for each context in the measurement cover $\mathcal{M}$.
By definition, a subset of observables in $\mathcal{X}$ that are not all included in a measurement context in $\mathcal{M}$ cannot be measured simultaneously. 
Therefore, a joint distribution over the said observables cannot be empirically estimated. 
The empirical model of a system fully encapsulates what is to be known from the system with empirical measurements. 

Contextuality comes from the failure of explaining an empirical model in a classically intuitive way: assuming that all measurements are just revealing deterministic pre-existing values, in other words, the measurement outcomes are already fixed when the system was prepared. 
Thus the randomness comes entirely from the system preparation. That means that there is a global joint distribution over all the observables in the scenario, which marginalises to every local joint distribution in the empirical model. 
Given an empirical model, if such a global distribution does not exist, then we call such empirical model contextual.
Note that such global distribution exists only in theory as there are observables in $ \mathcal{X}$ that cannot be measured simultaneously, unless in trivial scenarios. 

For readers familiar with sheaf theory, the said criterion for contextuality can be formalised using the language of sheaves. 
Consider the presheaf $ \mathcal{F}$ which assigns each subset $U \in \mathcal{P}(\mathcal{X})$ the set of all possible probability distributions on the observables in $U$. 
Each set inclusion $U \subseteq U'$, interpreted as an arrow in the category $\mathcal{P}(\mathcal{X})$, is mapped to the marginalisation of distributions on $U'$ to distributions on  $U$. 
For a measurement cover $ \mathcal{M}$, an empirical model is a family of pair-wise compatible distributions $ \{D_C\}_{C \in \mathcal{M}} $.
A necessary condition for a presheaf $\mathcal{F}$ to be a sheaf is the gluing property: 
\begin{quote}
    If $\mathcal{M} = \{C_i\}_i$ is a cover of $ \mathcal{U} \subseteq \mathcal{X}$, and $\{D_C\}_{C\in \mathcal{M}}$ is a family of pair-wise compatible sections, then there exists a global distribution $D$ on $ \mathcal{U}$ that marginalises to every $D_C$.
\end{quote}
Since a contextual empirical model on a measurement scenario entails that there is no such global distribution, the presheaf $\mathcal{F}$ is not a sheaf.
In other words, a contextual empirical model can only live on a measurement scenario for which the presheaf $\mathcal{F}$ is not a sheaf.

As an example, the Bell-CHSH scenario involves two experimenters, Alice and Bob, who share between them a two-qubit quantum state. 
Alice is allowed to measure her part of the state with one of two incompatible observables, $ a_1$ and $ a_2$, which gives either $0$ or $1$ as the outcome.
Similarly Bob can choose to measure his part with observables $ b_1$ and $ b_2$.
Therefore, the Bell-CHSH measurement scenario is fully described with the following data: $\mathcal{X} = \{a_1, b_1, a_2, b_2\}$,  $\mathcal{M} = \{\{a_1,b_1\}, \{a_1, b_2\}, \{a_2,b_1\}, \{a_2, b_2\} \}$, and $ \mathcal{O} = \{0, 1\} $.  Notice that $\{a_1, a_2\} $ and $\{b_1, b_2\} $ are not in $ \mathcal{M}$ as they cannot be measured simultaneously due to their quantum mechanical incompatibility.

So far we have specified what measurements are allowed and what outcomes are possible.
Suppose now Alice and Bob repeat the experiment many times and have gathered sufficient statistics to estimate the joint probability distributions for each context in $ \mathcal{M}$. Their results can be summarised in a table referred to as an \emph{empirical table}, see Figure \ref{fig:emptables}, where each row in the table represent a joint distribution on the context shown in the leftmost column. 
For instance, the bottom right entry in the table ($1 / 8$) is the probability of both Alice and Bob getting 1 as their measurement outcomes when Alice chooses to measure $ a_2$ and Bob chooses to measure $ b_2$.
Note that the empirical model of the system is entirely described by the empirical table.

\begin{figure}
\begin{center}
\begin{tabular}{r|ccccc}
 & $(0, 0)$ & $(0, 1)$ & $(1, 0)$ & $(1, 1)$   \\ \hline
$(a_1, b_1)$ & $1 / 2$ & $0$ & $0$ & $1 / 2$  \\
$(a_1, b_2)$ & $3 / 8$ & $1 / 2$ & $1 / 2$ & $3 / 8$  \\
$(a_2, b_1)$ & $3 / 8$ & $1 / 8$ & $1 / 8$ & $3 / 8$  \\
$(a_2, b_2)$ & $1 / 8$ & $3 / 8$ & $3 / 8$ & $1 / 8$  \\
\end{tabular}
\qquad
\begin{tabular}{r|ccccc}
 & $(0, 0)$ & $(0, 1)$ & $(1, 0)$ & $(1, 1)$   \\ \hline
$(a_1, b_1)$ & $1$ & $0$ & $0$ & $1$  \\
$(a_1, b_2)$ & $1$ & $1$ & $1$ & $1$  \\
$(a_2, b_1)$ & $1$ & $1$ & $1$ & $1$  \\
$(a_2, b_2)$ & $1$ & $1$ & $1$ & $1$  \\
\end{tabular}
\end{center}
\caption{Empirical tables of measurement scenarios:  Bell-CHSH (left), possibilistic Bell-CHSH  (right)}
\label{fig:emptables}
\end{figure}

One can show that, using elementary linear algebra, there exists no global distribution over $ \{a_1, a_2, b_1, b_2\} $ that marginalises to the 4 local distribution shown in the above empirical table.
Therefore, the empirical model considered here is indeed contextual.

Instead of probability, one can also consider possibility, i.e.\ whether an outcome is possible or not. 
If we use Boolean values to represent possibility, 0 for \emph{impossible} and 1 for \emph{possible}, the passage from probability to possibility is just a mapping of all zero probabilities to 0 and all non-zero probabilities to 1.
This (irreversible) mapping is called a \emph{possibilistic collapse} of the model.  For the empirical table of the possibilistic version of  Bell-CHSH see Figure \ref{fig:emptables}.  One can  visualise a possibilistic model with a bundle diagram, see Figure \ref{fig:bundles}:

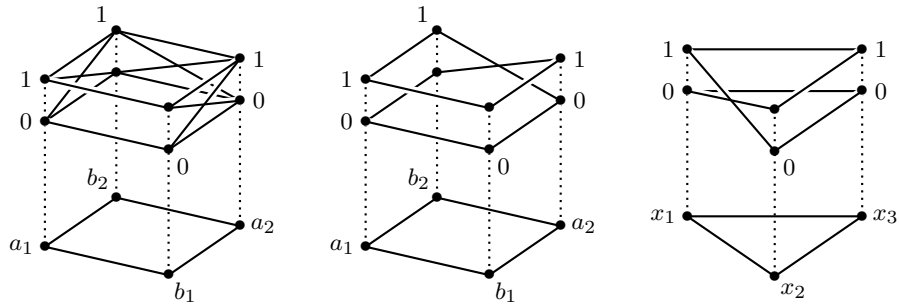
\begin{figure}
\begin{center}
\begin{tikzpicture}[x=45pt,y=45pt,thick,label distance=-0.25em,baseline=(O.base), scale=0.7]
\coordinate (O) at (0,0);
\coordinate (T) at (0,1.5);
\coordinate (u) at (0,0.5);
\coordinate [inner sep=0em] (v0) at ($ ({-cos(1*pi/12 r)*1.2},{-sin(1*pi/12 r)*0.48}) $);
\coordinate [inner sep=0em] (v1) at ($ ({-cos(7*pi/12 r)*1.2},{-sin(7*pi/12 r)*0.48}) $);
\coordinate [inner sep=0em] (v2) at ($ ({-cos(13*pi/12 r)*1.2},{-sin(13*pi/12 r)*0.48}) $);
\coordinate [inner sep=0em] (v3) at ($ ({-cos(19*pi/12 r)*1.2},{-sin(19*pi/12 r)*0.48}) $);
\coordinate [inner sep=0em] (v0-0) at ($ (v0) + (T) $);
\coordinate [inner sep=0em] (v0-1) at ($ (v0-0) + (u) $);
\coordinate [inner sep=0em] (v1-0) at ($ (v1) + (T) $);
\coordinate [inner sep=0em] (v1-1) at ($ (v1-0) + (u) $);
\coordinate [inner sep=0em] (v2-0) at ($ (v2) + (T) $);
\coordinate [inner sep=0em] (v2-1) at ($ (v2-0) + (u) $);
\coordinate [inner sep=0em] (v3-0) at ($ (v3) + (T) $);
\coordinate [inner sep=0em] (v3-1) at ($ (v3-0) + (u) $);
\draw (v0) -- (v1) -- (v2) -- (v3) -- (v0);
\draw [dotted] (v0-1) -- (v0);
\draw [dotted] (v1-1) -- (v1);
\draw [dotted] (v2-1) -- (v2);
\draw [dotted] (v3-1) -- (v3);
\node [inner sep=0.1em] (v0') at (v0) {$\bullet$};
\node [anchor=east,inner sep=0em] at (v0'.west) {$a_1$};
\node [inner sep=0.1em,label={[label distance=-0.625em]330:{$b_1$}}] at (v1) {$\bullet$};
\node [inner sep=0.1em] (v2') at (v2) {$\bullet$};
\node [anchor=west,inner sep=0em] at (v2'.east) {$a_2$};
\node [inner sep=0.1em,label={[label distance=-0.5em]175:{$b_2$}}] at (v3) {$\bullet$};

\draw [line width=3.2pt,white] (v3-0) -- (v0-0);
\draw [line width=3.2pt,white] (v3-0) -- (v0-1);
\draw [line width=3.2pt,white] (v3-1) -- (v0-0);
\draw [line width=3.2pt,white] (v3-1) -- (v0-1);
\draw (v3-0) -- (v0-0);
\draw (v3-0) -- (v0-1);
\draw (v3-1) -- (v0-0);
\draw (v3-1) -- (v0-1);

\draw [line width=3.2pt,white] (v2-0) -- (v3-0);
\draw [line width=3.2pt,white] (v2-0) -- (v3-1);
\draw [line width=3.2pt,white] (v2-1) -- (v3-0);
\draw [line width=3.2pt,white] (v2-1) -- (v3-1);
\draw (v2-0) -- (v3-0);
\draw (v2-0) -- (v3-1);
\draw (v2-1) -- (v3-0);
\draw (v2-1) -- (v3-1);

\draw [line width=3.2pt,white] (v0-0) -- (v1-0);
\draw [line width=3.2pt,white] (v0-1) -- (v1-1);
\draw (v0-0) -- (v1-0);
\draw (v0-1) -- (v1-1);

\draw [line width=3.2pt,white] (v1-0) -- (v2-0);
\draw [line width=3.2pt,white] (v1-0) -- (v2-1);
\draw [line width=3.2pt,white] (v1-1) -- (v2-0);
\draw [line width=3.2pt,white] (v1-1) -- (v2-1);
\draw (v1-0) -- (v2-0);
\draw (v1-0) -- (v2-1);
\draw (v1-1) -- (v2-0);
\draw (v1-1) -- (v2-1);

\node [inner sep=0.1em,label=left:{$0$}] at (v0-0) {$\bullet$};
\node [inner sep=0.1em,label=left:{$1$}] at (v0-1) {$\bullet$};
\node [inner sep=0.1em,label={[label distance=-0.5em]330:{$0$}}] at (v1-0) {$\bullet$};
\node [inner sep=0.1em] at (v1-1) {$\bullet$};
\node [inner sep=0.1em,label=right:{$0$}] at (v2-0) {$\bullet$};
\node [inner sep=0.1em,label=right:{$1$}] at (v2-1) {$\bullet$};
\node [inner sep=0.1em] at (v3-0) {$\bullet$};
\node [inner sep=0.1em,label={[label distance=-0.5em]150:{$1$}}] at (v3-1) {$\bullet$};
\end{tikzpicture}
\qquad
\begin{tikzpicture}[x=45pt,y=45pt,thick,label distance=-0.25em,baseline=(O.base), scale=0.7]
\coordinate (O) at (0,0);
\coordinate (T) at (0,1.5);
\coordinate (u) at (0,0.5);
\coordinate [inner sep=0em] (v0) at ($ ({-cos(1*pi/12 r)*1.2},{-sin(1*pi/12 r)*0.48}) $);
\coordinate [inner sep=0em] (v1) at ($ ({-cos(7*pi/12 r)*1.2},{-sin(7*pi/12 r)*0.48}) $);
\coordinate [inner sep=0em] (v2) at ($ ({-cos(13*pi/12 r)*1.2},{-sin(13*pi/12 r)*0.48}) $);
\coordinate [inner sep=0em] (v3) at ($ ({-cos(19*pi/12 r)*1.2},{-sin(19*pi/12 r)*0.48}) $);
\coordinate [inner sep=0em] (v0-0) at ($ (v0) + (T) $);
\coordinate [inner sep=0em] (v0-1) at ($ (v0-0) + (u) $);
\coordinate [inner sep=0em] (v1-0) at ($ (v1) + (T) $);
\coordinate [inner sep=0em] (v1-1) at ($ (v1-0) + (u) $);
\coordinate [inner sep=0em] (v2-0) at ($ (v2) + (T) $);
\coordinate [inner sep=0em] (v2-1) at ($ (v2-0) + (u) $);
\coordinate [inner sep=0em] (v3-0) at ($ (v3) + (T) $);
\coordinate [inner sep=0em] (v3-1) at ($ (v3-0) + (u) $);
\draw (v0) -- (v1) -- (v2) -- (v3) -- (v0);
\draw [dotted] (v0-1) -- (v0);
\draw [dotted] (v1-1) -- (v1);
\draw [dotted] (v2-1) -- (v2);
\draw [dotted] (v3-1) -- (v3);
\node [inner sep=0.1em] (v0') at (v0) {$\bullet$};
\node [anchor=east,inner sep=0em] at (v0'.west) {$a_1$};
\node [inner sep=0.1em,label={[label distance=-0.625em]330:{$b_1$}}] at (v1) {$\bullet$};
\node [inner sep=0.1em] (v2') at (v2) {$\bullet$};
\node [anchor=west,inner sep=0em] at (v2'.east) {$a_2$};
\node [inner sep=0.1em,label={[label distance=-0.5em]175:{$b_2$}}] at (v3) {$\bullet$};

\draw [line width=3.2pt,white] (v3-0) -- (v0-0);
\draw [line width=3.2pt,white] (v3-1) -- (v0-1);
\draw (v3-0) -- (v0-0);
\draw (v3-1) -- (v0-1);

\draw [line width=3.2pt,white] (v2-0) -- (v3-1);
\draw [line width=3.2pt,white] (v2-1) -- (v3-0);
\draw (v2-0) -- (v3-1);
\draw (v2-1) -- (v3-0);

\draw [line width=3.2pt,white] (v0-0) -- (v1-0);
\draw [line width=3.2pt,white] (v0-1) -- (v1-1);
\draw (v0-0) -- (v1-0);
\draw (v0-1) -- (v1-1);

\draw [line width=3.2pt,white] (v1-0) -- (v2-0);
\draw [line width=3.2pt,white] (v1-1) -- (v2-1);
\draw (v1-0) -- (v2-0);
\draw (v1-1) -- (v2-1);

\node [inner sep=0.1em,label=left:{$0$}] at (v0-0) {$\bullet$};
\node [inner sep=0.1em,label=left:{$1$}] at (v0-1) {$\bullet$};
\node [inner sep=0.1em,label={[label distance=-0.5em]330:{$0$}}] at (v1-0) {$\bullet$};
\node [inner sep=0.1em] at (v1-1) {$\bullet$};
\node [inner sep=0.1em,label=right:{$0$}] at (v2-0) {$\bullet$};
\node [inner sep=0.1em,label=right:{$1$}] at (v2-1) {$\bullet$};
\node [inner sep=0.1em] at (v3-0) {$\bullet$};
\node [inner sep=0.1em,label={[label distance=-0.5em]150:{$1$}}] at (v3-1) {$\bullet$};
\end{tikzpicture}
\qquad
\begin{tikzpicture}[x=45pt,y=45pt,thick,label distance=-0.25em,baseline=(O.base), scale=0.7]
\coordinate (O) at (0,0);
\coordinate (T) at (0,1.5);
\coordinate (u) at (0,0.5);
\coordinate [inner sep=0em] (v0) at ($ ({-cos(-2*pi/12 r)*1.2},{-sin(-2*pi/12 r)*0.48}) $);
\coordinate [inner sep=0em] (v1) at ($ ({-cos(6*pi/12 r)*1.2},{-sin(6*pi/12 r)*0.48}) $);
\coordinate [inner sep=0em] (v2) at ($ ({-cos(14*pi/12 r)*1.2},{-sin(14*pi/12 r)*0.48}) $);
\coordinate [inner sep=0em] (v0-0) at ($ (v0) + (T) $);
\coordinate [inner sep=0em] (v0-1) at ($ (v0-0) + (u) $);
\coordinate [inner sep=0em] (v1-0) at ($ (v1) + (T) $);
\coordinate [inner sep=0em] (v1-1) at ($ (v1-0) + (u) $);
\coordinate [inner sep=0em] (v2-0) at ($ (v2) + (T) $);
\coordinate [inner sep=0em] (v2-1) at ($ (v2-0) + (u) $);
\draw (v0) -- (v1) -- (v2) -- (v0);
\draw [dotted] (v0-1) -- (v0);
\draw [dotted] (v1-1) -- (v1);
\draw [dotted] (v2-1) -- (v2);

\node [inner sep=0.1em] (v0') at (v0) {$\bullet$};
\node [anchor=east,inner sep=0em] at (v0'.west) {$x_1$};
\node [inner sep=0.1em,label={[label distance=-0.625em]330:{$x_2$}}] at (v1) {$\bullet$};
\node [inner sep=0.1em,label={[label distance=-0.0em]0:{}}] (v2') at (v2) {$\bullet$};
\node [anchor=west,inner sep=0em] at (v2'.east) {$x_3$};

\draw [line width=3.2pt,white] (v0-1) -- (v2-1);
\draw [line width=3.2pt,white] (v0-0) -- (v2-0);
\draw (v0-1) -- (v2-1);
\draw (v0-0) -- (v2-0);

\draw [line width=3.2pt,white] (v2-1) -- (v1-1);
\draw [line width=3.2pt,white] (v2-0) -- (v1-0);
\draw (v2-1) -- (v1-1);
\draw (v2-0) -- (v1-0);

\draw [line width=3.2pt,white] (v1-1) -- (v0-0);
\draw [line width=3.2pt,white] (v1-0) -- (v0-1);
\draw (v1-1) -- (v0-0);
\draw (v1-0) -- (v0-1);

\node [inner sep=0.1em,label=left:{$0$}] at (v0-0) {$\bullet$};
\node [inner sep=0.1em,label=left:{$1$}] at (v0-1) {$\bullet$};
\node [inner sep=0.1em,label={[label distance=-0.5em]330:{$0$}}] at (v1-0) {$\bullet$};
\node [inner sep=0.1em] at (v1-1) {$\bullet$};
\node [inner sep=0.1em,label=right:{$0$}] at (v2-0) {$\bullet$};
\node [inner sep=0.1em,label=right:{$1$}] at (v2-1) {$\bullet$};
\end{tikzpicture}
\end{center}
\caption{Bundle diagrams of possibilistic empirical models: (left) the Bell-CHSH test; (middle) a PR box; (right) a PR prism.}
\label{fig:bundles}
\end{figure}
The base of the bundle diagram represents the measurement cover $\mathcal{M}$. 
An edge is drawn between two observables if they can be simultaneously measured, i.e.\ in the same measurement context. 
What sits on top of the base represents the possible outcomes. For instance, the presence of the edge connecting the 0 vertex on top of $ a_1$ and the 0 vertex on the observable $ b_1$ means that it is possible to get the joint outcome $(0, 0)$ when the context  $(a_1, b_1) $ is measured.

A system is logically contextual if the inexistence of a global distribution can already be deduced by looking at the supports of the context-wise distributions -- or
equivalently if the Boolean distributions obtained by the \emph{possibilistic
collapse} of the model \cite{Abramsky2011} is contextual. Such systems are said to be \emph{possibilistically contextual}.  
Logical contextuality manifests on a bundle diagram as the failure of extending every edge to a univocal loop (a loop that wrap around the base once).  For the possibilistic empirical model of a PR box, see Figure \ref{fig:bundles}. 

Note that none of the edges is extendable to a loop that wraps around the base once. Given the possibilistic collapse of an empirical model, if none of the edges can be extendable to a loop that wraps around the base once, we say that the model is \emph{strongly contextual}. \footnote{Strictly speaking, this definition of strong contextuality only applies to cyclic scenarios where the base of the bundle diagram forms a loop. Nonetheless, cyclic scenarios are the only scenarios considered in this paper.}

\begin{Definition}[Minimally contextual scenario]
    A measurement scenario $\langle X, M, O \rangle$ is called \emph{minimally contextual} if
    \begin{enumerate}
        \item it admits a contextual empirical model, and
        \item removing any one of the observables, contexts or outcomes would make the scenario lose the ability to admit a contextual empirical model.
    \end{enumerate}
\end{Definition}

\begin{Definition}[$k$-cyclic scenario]
    A $k$-cyclic scenario is a measurement scenario $\langle X, M, O\rangle$ such that $|X| = |M| = k$ and $\forall c \in M, |c| = 2$.
    The observables and contexts of a $k$-cyclic scenario can always be written as 
    \begin{enumerate}
        \item $X = \{x_1, x_2, \dots, x_k\}$,
        \item $M = \{\{x_1, x_2\}, \{x_2, x_3\}, \dots, \{x_k, x_1\}\}$.
    \end{enumerate}
\end{Definition}

\begin{Proposition}
The 3-cyclic scenario with binary outcomes is minimally contextual.
\end{Proposition}
\begin{proof}
    The 3-cyclic scenario has the following data:
    \begin{enumerate}
        \item $\mathcal{X} = \{x_1, x_2, x_3\}$,
        \item $\mathcal{M} = \{\{x_1, x_2\}, \{x_2, x_3\},  \{x_1, x_3\}\}$,
        \item $\mathcal{O} = \{0,1\}$.
    \end{enumerate}
    The PR prism is a contextual empirical model for this scenario.
    For minimality:
    \begin{enumerate}
        \item Removing any one of the observables would render two of the contexts to be trivial and thus the scenario would lose the ability to admit a contextual empirical model.
        \item Removing a context would render the scenario to be non-cyclic and thus the scenario would lose the ability to admit a contextual empirical model~\cite{soares2014}.
        \item Removing an outcome of one of the observables would render that observable to be trivial which is equivalent to the removal of an observable.
    \end{enumerate}
\end{proof}

\begin{Proposition} The only strongly contextual system, up to relabelling, for the minimal measurement scenario is where perfect correlation is observed on two of the contexts and perfect anti-correlation is observed on the other one. 
\end{Proposition}

We call this scenario the \emph{PR prism} as an analogy to the PR box, see Figure \ref{fig:bundles} for its bundle diagram. The pairs of parallel edges over contexts $\{x_2, x_3\} $ and $\{x_3, x_1\} $ correspond to perfect correlation and the pair of crossed edges over context $ \{x_1, x_2\} $ corresponds to the perfect anti-correlation.

\subsection{Contextual and Signalling Fractions}

The contextual fraction ($\CF$) \cite{Abramsky2017} measures the degree of contextuality of a given no-signalling model.
Given an empirical model $e$, the $\CF$ of $e$ is defined as the minimum $\lambda$ such that the following convex decomposition of $e$ works:
\begin{equation}
    \label{eq:convexcf}
    e = (1-\lambda) e^{NC} + \lambda e^{C},
\end{equation}
where $e^{NC}$ is a non-contextual (and no-signalling) empirical model and  $e^{C}$ is a model allowed to be contextual. 
For no-signalling models, the criterion of contextuality is just
\begin{equation}
    \label{eq:cfnosig}
   \CF > 0.
\end{equation}
As $e^{NC}$ is not allowed to be signalling, the $\CF$ of a signalling model must be greater than zero.
Thus, interpreting $\CF$ as a measure of contextuality for signalling models
would lead to erroneous conclusions. 
However, most models, including the ones considered in this paper, are signalling.

One can try to define a signalling fraction ($\SF$), in the same way $\CF$ is defined, to quantity the degree of signalling. Given a model $e$, the $\SF$ of $e$ is defined as the minimum $\mu$ such that the following convex decomposition of  $e$ works:
\begin{equation}
    \label{eq:convexsf}
    e = (1-\mu) e^{NS} + \mu e^{S},
\end{equation}
where $e^{NS}$ is a no-signalling empirical model and  $e^{S}$ is a model allowed to be signalling.

In \cite{Emeriau2022}, the signalling fraction ($\SF$) was used to quantify the amount of fictitious contextuality contributing to the contextual fraction due to signalling when applied to a signalling model.
The authors derived a criterion of contextuality for signalling models that reads 
\begin{equation}
\label{eq:cf}
\CF > 2 |\cal M| \, \SF
,\end{equation}
where $|\cal M|$ denotes the number of measurement contexts. Notice how criterion (\ref{eq:cf}) reduces to the generalised criterion (\ref{eq:cfnosig}) when $\SF = 0$.

\subsection{PR-like models}
In general, the calculation of $\CF$ and $\SF$ requires solving linear programs.
However, if the model has certain structure, the calculation could become much simpler.

We define a class of models called the \emph{PR-like} models of the Bell-CHSH scenario, which are models that share the same support as a PR box. 
More generally, a PR-like model of a $k$-cyclic scenario is a model that shares the same support as a PR box of the same scenario.

For example, a rank 3 PR-like model can be parametrised as follows:

\begin{center}
\begin{tabular}{r|ccccc}
 & $(0, 0)$ & $(0, 1)$ & $(1, 0)$ & $(1, 1)$   \\ \hline
$(x_1, x_2)$ & $\frac{1 + \epsilon_1}{2}$ & $0$ & $0$ & $\frac{1 - \epsilon_1}{2}$  \\
$(x_2, x_3)$ & $\frac{1 + \epsilon_2}{2}$ & $0$ & $0$ & $\frac{1 - \epsilon_2}{2}$  \\
$(x_3, x_1)$ & $0$ & $\frac{1 + \epsilon_3}{2}$ & $\frac{1 - \epsilon_3}{2}$ & $0$
\end{tabular}
\end{center}
where $\epsilon_1, \epsilon_2, \epsilon_3 \in [-1, 1]$.
We show that the $\CF$ and $\SF$ of a PR-like model can be computed analytically without solving the linear programs numerically.
\begin{Proposition}
    \label{prop:cfprlike}
    The contextual fraction of a PR-like model is always $1$.
\end{Proposition}
\begin{proof}
    Consider a PR-like model $e$ with a convex decomposition as in (\ref{eq:convexcf}):
    \begin{equation*}
        e = (1-\lambda) e^{NC} + \lambda e^{C},
    \end{equation*}
    where $e^{NC}$ is a non-contextual and no-signalling model and $e^{C}$ is a model allowed to be contextual and signalling.
    Recall that the contextual fraction of $e$ is the minimum $\lambda$ such that the above decomposition remains valid.
    As the coefficents $\lambda$ and $(1 - \lambda)$ are non-negative, the models $e^{NC}$ and $e^{C}$ are also PR-like models, otherwise the decomposition would not be valid.
    It is routine to show that the only no-signalling PR-like model is the PR box, which is known to be strongly contextual. Hence there does not exist a valid $e^{NC}$ for the decomposition. 
    Therefore, the minimum $\lambda$ is $1$ and the $\CF$ of a PR-like model is always $1$.
\end{proof}

\begin{Proposition}
    \label{prop:sfprlike}
    The signalling fraction of a PR-like model is given by $max|\epsilon_i|$.
\end{Proposition}
\begin{proof}
    Consider a PR-like model $e$ with a convex decomposition as in (\ref{eq:convexsf}):
    \begin{equation*}
        e = (1-\mu) e^{NS} + \mu e^{S},
    \end{equation*}
    where $e^{NS}$ is a no-signalling model and $e^{S}$ is a model allowed to be signalling. 
    The signalling fraction of $e$ is the minimum $\mu$ such that the above decomposition remains valid.
    As per the proof of Proposition \ref{prop:cfprlike}, the model no-signalling model $e^{NS}$ is PR-like and thus a PR box.
    While minimising $\mu$ (maximising $1 - \mu$), we have to make sure that
    \begin{equation*}
        e \ge (1-\mu) e^{NS},
    \end{equation*}
    so that the decomposition remains valid.
    The minimum $\mu$ is achieved when the inequality is tight, i.e.\ when 
    \begin{equation*}
        (1 - \mu_{\text{min}}) \frac{1}{2} = \min_{i \in \{1,2,3\}} \left(\frac{1 \pm \epsilon_i}{2}\right).
    \end{equation*}
    Hence the signalling fraction of a PR-like model is given by $\max_i|\epsilon_i|$.
\end{proof}

The two propositions above showed that the values of $\CF$ and $\SF$ of a PR-like model can be computed analytically without resorting the linear programming. Additionally, the criteria (\ref{eq:cf}) can be specialised to PR-like models as follows:
\begin{equation}
    \label{eq:cfprlike}
    \SF \le \frac{1}{2 M},
\end{equation}
where M is the number of measurement contexts.

\section{Background on the Contextuality by Default Framework}
\label{sec:cbd}
In the setting of Contextuality-by-Default (CbD), there are two important notions: \emph{contents}, denoted by $q_i$, which are measurements, or more generally, questions about the system; and \emph{contexts}, denoted by $c^j$, which represent the conditions under which the questions are asked, \ their ordering. 
Every $q_i$ in a $c^j$ gives rise to a random variable $R^j_i$ taking values in $\{\pm 1\}$, and representing possible answers and their probabilities. All random variables in a given context are jointly distributed. 

A well-studied class of CbD systems are the cyclic systems~\cite{Dzhafarov2013,Dzhafarov2016a,Dzhafarov2015a}, where each context has exactly 2 contents and every content is in exactly 2 contexts.
The rank of a cyclic system is the number of contents, or equivalently, the number of contexts.

A cyclic system of rank $n$ is contextual if and only if $\CNT{1}$ is positive, where $\CNT{1}$ is defined as:
\begin{equation}\label{eq:BellInequality}
    \CNT{1} := s_{odd} \left(\left\{\left<R^{j}_{i_j}R^{j}_{i'_j}\right>\right\}_{j=1,\ldots,n}\right) - \Delta - n + 2 > 0
\end{equation}
where $i_j\neq i'_j$ for all $j$ and $R^{j}_{i_j}, R^{j}_{i'_j}$ are well-defined for all $j$. Quantities $s_{odd}: \mathbb{R}^n \to \mathbb{R}$ and $\Delta$ are defined as follows:
\begin{equation}
    s_{odd}\left(\underline{x}\right) = \max_{\substack{\underline{\sigma}\in \{\pm1\}^k; \\ \mathfrak{p}(\underline{\sigma}=-1)}}\underline{\sigma}\cdot \underline{x}\ ; \qquad 
\Delta = \sum_{i=1}^n \left|\left<R^{j_i}_{i}\right> - \left<R^{j'_i}_{i}\right>\right|
\end{equation}
where $\mathfrak{p}(\underline{\sigma}) = \prod_{i=1}^n \sigma_i$ ($\mathfrak{p}$ is the parity function of $\underline{\sigma}$).  
The quantity $\Delta$ is called \emph{direct influence} which measures the degree of signalling in the system. A no-signalling system has $\Delta=0$. 

A feature of the $\CNT{1}$ measure is that it generalises the Bell-CHSH inequaltiy. To see this, consider a cyclic system of rank 4, i.e.\ the Bell-CHSH scenario. The $\CNT{1}$ quantity becomes:
\begin{align}
    \CNT{1} = &s_\text{odd}\left(\left< R^0_0\ R^0_1 \right>, \left< R^1_1\ R^1_2 \right>, \left< R^2_2\ R^2_3 \right>, \left< R^3_3\ R^3_0 \right>\right) \\ \nonumber
    &- \left|\left< R^0_1 \right> - \left< R^1_1 \right>\right| - \left|\left< R^1_2 \right> - \left< R^2_2 \right>\right| - \left|\left< R^2_3 \right> - \left< R^3_3 \right>\right| - \left|\left< R^3_0 \right> - \left< R^0_0 \right>\right| - 2.
\end{align}
Note the Bell-CHSH inequality is recovered by the measure $\CNT{1}$ in the case of no-signalling systems, i.e.\ $\Delta = 0$.

\subsection*{PR-like models in the CbD framework}

For a PR-like model, the correlation term $\left< R^j_{i}R^j_{i+1} \right>$ is either $+1$ for the correlated contexts or $-1$ for the anti-correlated contexts.
By definition, a PR-like model has odd number of anti-correlated contexts, thus the values of the $s_\text{odd}$ term is always $n$.

The quantity $\Delta$ is the sum of the absolute differences between the correlations of the same observable in different contexts. For a PR-like model, $\Delta$ has the following form:
\begin{equation*}
    \Delta = |\epsilon_1 - \epsilon_2| + |\epsilon_2 - \epsilon_3| + \cdots + |\epsilon_{n-1} - \epsilon_n| + |\epsilon_n + \epsilon_1|
\end{equation*}
Recall that for PR-like models we have $\SF = \max_i |\epsilon_i|$, which can be used to bound the value of $\Delta$.
One can show that for a PR-like model with $n$ contexts, $\Delta$ is bounded by:
\begin{equation*}
    2 \SF \leq \Delta \leq \begin{cases}
        2n \SF &\quad n\text{ odd} \\
        2(n-1) \SF &\quad n\text{ even}
    \end{cases}.
    \label{eq:inequality}
\end{equation*}
The intuition behind the lower bound is that the value of $\SF$ makes sure that there must be a $\epsilon_i$ that is at least $\SF$ away from $0$.
The relationship between $\Delta$ and $\SF$ is illustrated in Figure \ref{fig:sfdeltaheatmap}. 
One can readily see that the sheaf contextual region is a strict subset of the CbD contextual region (considering only the subspace where empirical models are allowed), unless 
the bound established by (\ref{eq:cfprlike}) is extended to no less than $2/6$.

\begin{figure}[b!]
\begin{center}
    \fbox{\begin{minipage}{9cm}{
        There is an \emph{apple} and a \emph{strawberry}. 
    \begin{enumerate}
        \item One of them is  \emph{red} and the very same one is \emph{round}. 
        \item  One of them is  \emph{round} and that very same one is  \emph{sweet}. 
        \item One of them is   \emph{sweet} and the other one of them is \emph{red}.
    \end{enumerate}}\end{minipage}}
    \end{center}
    \caption{Example of the PR-anaphora schema with adjectival modifiers.}
    \label{fig:schemadj}
\end{figure}

\begin{figure}[t!]
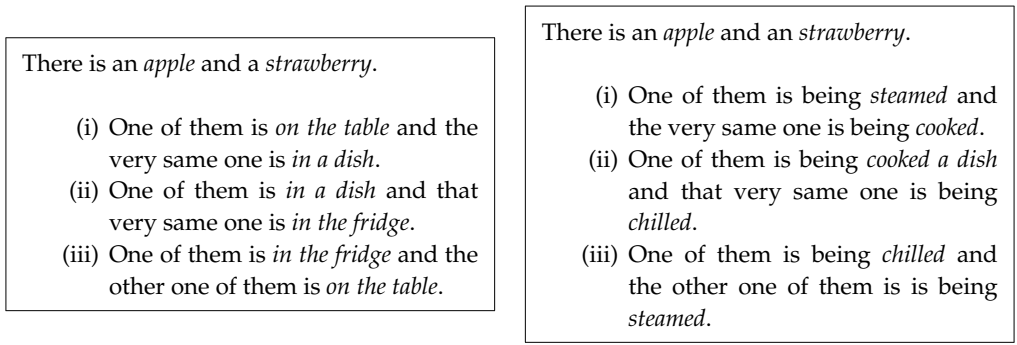

\centering
\fbox{\begin{minipage}{6cm}
        There is an \emph{apple} and a \emph{strawberry}. 
    \begin{enumerate}
        \item One of them is  \emph{on the table} and the very same one is \emph{in a dish}. 
        \item One of them is  \emph{in a dish} and that very same one is \emph{in the fridge}. 
        \item One of them is  \emph{in the fridge} and the other one of them is \emph{on the table}.
    \end{enumerate}\end{minipage}}
    \quad
    \fbox{\begin{minipage}{6cm}
        There is an \emph{apple} and an \emph{strawberry}. 
    \begin{enumerate}
        \item One of them is being \emph{steamed} and the very same one is being \emph{cooked}. 
        \item One of them is being \emph{cooked a dish} and that very same one is being \emph{chilled}. 
        \item One of them is being \emph{chilled} and the other one of them  is is being \emph{steamed}.
    \end{enumerate}
\end{minipage}}
\caption{Examples of the PR-anaphora schema with verbs (left) and preposition modifiers (right)}
\label{fig:Prism}
\end{figure}

\section{Methodology: PR-anaphora schema}
\label{sec:pos}
The goal is to construct a schema with anaphoric ambiguities that can be modelled by a measurement scenario capable of admitting contextual empirical models.
For simplicity, we use the minimal measurement scenario, i.e.\ the 3-cyclic scenario, as the underlying measurement scenario and design the schema in such a way that the resulting empirical model is as close to a it  as possible. The measurement scenario is as follows:
\begin{enumerate}
    \item observables $X = \{X_1, X_2, X_3\}$,
    \item contexts $\mathcal{M} = \{\{X_1, X_2\}, \{X_2, X_3\}, \{X_3, X_1\}\}$,
    \item outcomes $O = \{O_1, O_2\}$.
\end{enumerate}
Here the ambiguity lies in the anaphors $X_1$, $X_2$ and $X_3$ and the possible referents are $O_1$ and $O_2$. 

\begin{Definition}
    The PR-anaphora schema is defined as follows: 
    \begin{center}
        \begin{minipage}{7cm}
        There is an $O_1$ and an $O_2$.
        \begin{enumerate}
            \item It is $ X_1$ and the same one is $ X_2$. 
            \item It is $ X_2$ and the same one is $ X_3$. 
            \item It is $ X_3$ and the other one is $ X_1$.
        \end{enumerate}
        \end{minipage}
    \end{center}
Here $O_1$ and $O_2$ are two noun phrases, i.e.\ the candidate referents; $X_1, X_2, X_3$ are three modifiers commonly acting on  $ O_1, O_2$. The $X_i$'s are the observables of the scenario. 
\label{def:pranaphora}
\end{Definition}

It follows that the schema can be possibilistically  modelled by a PR-prism and  is logically contextual.  The  schema is however  generated  in a way such that it is easily readable by computers.  In its current form, it is grammatical but unnatural; in the sense that  it   could  not have been generated by humans.  Figures \ref{fig:schemadj} and \ref{fig:Prism}  show some natural instantiations of it to  nouns and their  adjectival,  verb,  and  preposition modifiers. Other modifiers can be dealt with in a similar fashion. 

\subsection{Probabilistic PR-anaphora schema}
To construct probabilistic models for the PR-anaphora schema, we need to define a probability distribution over the possible referents of the anaphors.
We do this by leveraging the masked  language model BERT, which has provided improved baselines for many NLP tasks  \cite{Devlin2019}.
In practice, the masked word is replaced with a special token \texttt{[MASK]} and the model produces a probability distribution over the entire vocabulary used by BERT.
For example, given a sentence such as: {\small \texttt{The goal of life is [MASK].}}, BERT produces a probability distribution over each word in the vocabulary:
\begin{center}
\begin{tabular}{r|cccccc}
token   & \texttt{life} & \texttt{survival} & \texttt{love} & \texttt{freedom} & \texttt{simplicity} & $\cdots$ \\
\midrule
    prob. & 0.1093 & 0.0394 & 0.0329 & 0.0300 & 0.0249 & $\cdots$
\end{tabular}   
\end{center}

In order to  construct probabilistic models for the PR-anaphora schema, we go through these two steps:   first, as the ambiguities lies in the anaphors, we replace the anaphor with the special token \texttt{[MASK]}, then the  prediction of BERT is  interpreted as the probability distribution over the possible referents of the anaphor. As an example consider the following 3 sentences. We feed them separately to BERT:


\begin{enumerate}
    \item {There is an apple and a strawberry.} {The [MASK] is red and the same one is round.}
    \item {There is an apple and a strawberry.} {The [MASK] is round and the same one is sweet.}
    \item {There is an apple and a strawberry.} {The [MASK] is sweet and the other one is red.}
\end{enumerate}


\noindent BERT will  produce, probabilities $P_i\left( \texttt{apple} \right) $ and $P_i\left( \texttt{strawberry} \right) $ for the i-th sentence shown above. As BERT gives a probability score to every word in the vocabulary which sum to one, it is unlikely that $P_i\left( \texttt{apple} \right) +P_i\left( \texttt{strawberry} \right) = 1$. 
We therefore normalise them using the following map \footnote{The normalisation here is equivalent to limiting the vocabulary to just \texttt{apple} and \texttt{strawberry} when BERT computes the probability scores.}:
\begin{align*}
    P_i\left( \texttt{apple} \right) \mapsto \frac{P_i\left( \texttt{apple} \right)}{P_i\left( \texttt{apple} \right) + P_i\left( \texttt{strawberry} \right)} \\
    P_i\left( \texttt{strawberry} \right) \mapsto \frac{P_i\left( \texttt{strawberry} \right)}{P_i\left( \texttt{apple} \right) + P_i\left( \texttt{strawberry} \right)}
\end{align*}
We  use the normalised probabilities to construct a PR-like model with empirical table:
\begin{center}
\begin{tabular}{r|cccc}
    & $(\text{app.},\text{app.})$ & $(\text{app.}, \text{str.})$ & $(\text{str.}, \text{app.})$ & $(\text{str.}, \text{str.})$   \\
    \midrule
    $(\text{red}, \text{round})$ & $P_1\left( \texttt{apple} \right) $ & $0$ & $0$ & $P_1\left( \texttt{strawberry} \right) $  \\
    $(\text{round}, \text{sweet})$ & $P_2\left( \texttt{apple} \right) $ & $0$ & $0$ & $P_2\left( \texttt{strawberry} \right) $  \\
    $(\text{sweet}, \text{red})$ & $0$ & $P_3\left( \texttt{apple} \right) $ & $P_3\left( \texttt{strawberry} \right) $ & $0$
\end{tabular}
\end{center}
This procedure can be used similarly for other modifiers such as verbs and prepositions.
Notice that such an empirical model is no-signalling only if $P_i\left( \texttt{apple} \right) = P_i\left( \texttt{strawberry} \right) = 0.5$ for all $i$.
It is therefore very unlikely that the model is no-signalling. To determine whether a signalling model is contextual, we use the inequality criterion of Equation (\ref{eq:cf}). One can show that the contextual fraction $\CF$ of a model that has the same support as the PR prism is always 1. 
Also, all the examples we considered in this paper have 3 contexts, i.e.\ $|\mathcal{M}| = 3$.
Thus, to tell if such a model is sheaf-contextual, we just need to check if $\SF < 1/6$.  
Since the value of $\SF$ is determined by the maximum absolute value of the $\epsilon_i$'s, we need to make sure that the $\epsilon_i$'s are as small as possible. In other words, we ought to make BERT as uncertain as possible about the masked word, ideally giving it a uniform distribution over the candidates.
%


\begin{figure}[t!]
    \footnotesize
        \centering
    \begin{subfigure}[b]{0.45\textwidth}
        \begin{tikzpicture}

\def\coordinatesone{
(0.00000, 89)
(0.01667, 541)
(0.03333, 1481)
(0.05000, 2829)
(0.06667, 4706)
(0.08333, 7014)
(0.10000, 9892)
(0.11667, 13011)
(0.13333, 16616)
(0.15000, 20939)
(0.16667, 20939)
}
\def\coordinatestwo{
(0.16667, 25401)
(0.18333, 30238)
(0.20000, 35510)
(0.21667, 41778)
(0.23333, 48285)
(0.25000, 55150)
(0.26667, 61378)
(0.28333, 69755)
(0.30000, 77316)
(0.31667, 86630)
(0.33333, 95831)
(0.35000, 105213)
(0.36667, 116141)
(0.38333, 127268)
(0.40000, 138735)
(0.41667, 150332)
(0.43333, 163467)
(0.45000, 176802)
(0.46667, 190485)
(0.48333, 206229)
(0.50000, 222590)
(0.51667, 239353)
(0.53333, 256921)
(0.55000, 275669)
(0.56667, 296825)
(0.58333, 319689)
(0.60000, 342735)
(0.61667, 365744)
(0.63333, 392241)
(0.65000, 420631)
(0.66667, 452127)
(0.68333, 485798)
(0.70000, 523872)
(0.71667, 564142)
(0.73333, 610601)
(0.75000, 659018)
(0.76667, 715863)
(0.78333, 781303)
(0.80000, 851162)
(0.81667, 935921)
(0.83333, 1034126)
(0.85000, 1156000)
(0.86667, 1298401)
(0.88333, 1486346)
(0.90000, 1722286)
(0.91667, 2044158)
(0.93333, 2523394)
(0.95000, 3312692)
(0.96667, 4975203)
(0.98333, 20594538)
(1.00000, 20594538)
}

\begin{semilogyaxis}[
    xmin=0, xmax=1,
    xtick={0,1/6,2/6,3/6,4/6,5/6,1},
    xticklabels={$0$,  $1 / 6$,  $2 / 6$,  $3 / 6$,  $4 / 6$,  $5 / 6$,  $1$},
    grid,
    grid style=dashed,
    legend pos={outer north east},
    legend cell align={left},
    ytick={1,10,100,1000,10000,100000,1000000,10000000},
    width=7cm,
    ]
\addplot+[ybar interval,mark=no,color=gray,fill=gray!50,ybar legend] plot coordinates {
\coordinatestwo
};
\addplot+[ybar interval,mark=no,color=black,fill=black!50,ybar legend] plot coordinates {
\coordinatesone
};
\end{semilogyaxis}
\end{tikzpicture}
        \caption{Signalling fraction}
        \label{fig:sfhisto}
    \end{subfigure}
    \hfill
    \begin{subfigure}[b]{0.45\textwidth}
        \begin{tikzpicture}

\def\coordinatesone{
(0.00000, 1405)
(0.10000, 9909)
(0.20000, 26559)
(0.30000, 51635)
(0.40000, 85610)
(0.50000, 127281)
(0.60000, 178814)
(0.70000, 241505)
(0.80000, 315771)
(0.90000, 404327)
(1.00000, 508979)
(1.10000, 634919)
(1.20000, 787258)
(1.30000, 975805)
(1.40000, 1225513)
(1.50000, 1558731)
(1.60000, 2039628)
(1.70000, 2841985)
(1.80000, 4509039)
(1.90000, 20414275)
(2.00000, 20414275)
}
\def\coordinatestwo{
(2.00000, 4538077)
(2.10000, 1490723)
(2.20000, 1090879)
(2.30000, 881894)
(2.40000, 743928)
(2.50000, 641829)
(2.60000, 566138)
(2.70000, 500857)
(2.80000, 448756)
(2.90000, 403185)
(3.00000, 363634)
(3.10000, 330001)
(3.20000, 299634)
(3.30000, 272873)
(3.40000, 247553)
(3.50000, 226047)
(3.60000, 205371)
(3.70000, 187851)
(3.80000, 170842)
(3.90000, 156100)
(4.00000, 142176)
(4.10000, 129847)
(4.20000, 118660)
(4.30000, 107623)
(4.40000, 96928)
(4.50000, 88631)
(4.60000, 79874)
(4.70000, 72337)
(4.80000, 64909)
(4.90000, 58254)
(5.00000, 52133)
(5.10000, 45802)
(5.20000, 40613)
(5.30000, 36058)
(5.40000, 31313)
(5.50000, 27016)
(5.60000, 23257)
(5.70000, 19359)
(5.80000, 15736)
(5.90000, 10834)
(6.00000, 10834)
}

\begin{semilogyaxis}[
    xmin=0, xmax=6,
    xtick={0,1,2,3,4,5,6},
    xticklabels={$0$, $1$, $2$, $3$, $4$, $5$, $6$},
    grid,
    grid style=dashed,
    legend pos={outer north east},
    legend cell align={left},
    ytick={1,10,100,1000,10000,100000,1000000,10000000},
    width=7cm,
    ]
\addplot+[ybar interval,mark=no,color=gray,fill=gray!50,ybar legend] plot coordinates {
\coordinatestwo
};
\addplot+[ybar interval,mark=no,color=black,fill=black!50,ybar legend] plot coordinates {
\coordinatesone
};
\end{semilogyaxis}
\end{tikzpicture}
        \caption{Direct influence}
        \label{fig:deltahisto}
    \end{subfigure}
    \caption{Histograms of (a) the signalling fraction and (b) the direct influence of the 51,966,480 models constructed for the full dataset
    Highlighted are the contextual models with (a) $\SF < 1/6$ or (b) $\Delta < 2$.
    The fraction of sheaf-contextual models is 0.148\% and the fraction of CbD-contextual models is 71.1\%.
    The sheaf-contextual bars in the histogram are too small to be visible.}

    \label{fig:histograms}
\end{figure}
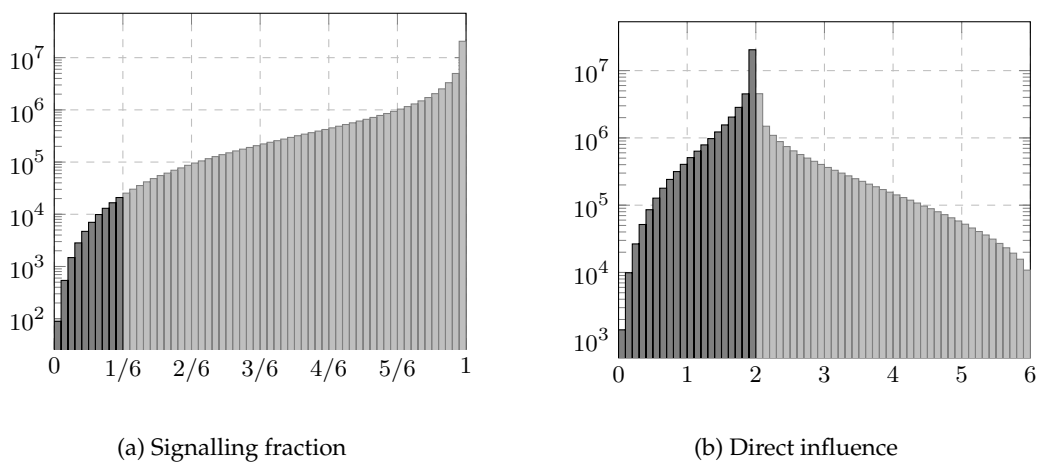

\section{Results}

\subsection{Dataset}
We adopt a systemic approach to construct a much larger dataset of empirical models in order to investigate the prevalence of contextuality in natural language data on a large scale.
To this end, we considered the entire Simple English Wikipedia corpus using a March 2022 snapshot of it which is made available to researchers.
This snapshot contains 205,328 articles with 40 million tokens in total, averaging 197 tokens per article.
The Simple English Wikipedia is a version of the English Wikipedia that uses a limited vocabulary and simpler grammar. It is designed for people less proficient in English, such as children and non-native speakers.
We chose this corpus because it is relatively small size while still containing a diverse range of topics.
We extracted all the adjective-noun phrases from the dataset and used them to construct examples of the PR-anaphora schema.
The dataset underwent the following standard preprocessing steps used in previous work \cite{LingMat2019}:
\begin{enumerate}
    \item Each article was tokenised using the \texttt{word\_tokenize} function in NLTK~\cite{bird2009natural}.
    \item Each tokenised article was then split into sentences using the \texttt{sent\_tokenize} function in NLTK.
    \item Each sentence was tagged with the Penn Treebank tag set using the \texttt{pos\_tag\_sents} function in NLTK to obtain the part-of-speech tags for each token. We used the default tagger offered by NLTK, which was a Greedy Averaged Perceptron tagger.
    \item Adjective-noun phrases were extracted from the whole tokenised dataset by scanning through the part-of-speech tags for each sentence. Neighbouring  tokens with the tags \texttt{JJ} and \texttt{NN} were extracted as adjective-noun phrases.
\end{enumerate}

After filtering out noise words, e.g.\ one letter words and numbers, or nouns that were not in the BERT vocabulary,  
%
we obtained  219,633 adjective-noun phrases, 9,521 nouns, and 21,152 adjectives.  To construct examples of the PR-anaphora schema, we chose the 5 most frequent common adjectives for the noun pair. The number was taken to be 5 since this provided us with a good level of overlap and at the same time, a large amount of data. This resulted in 866,108  noun pairs, with which we constructed 51,966,480 examples of the PR-anaphora schema.
Samples of the noun pairs and their corresponding adjectives are shown in Table \ref{tab:instanceswithsfdelta}.

\subsection{Contextuality in the dataset}
We used BERT to construct empirical models for all 51,966,480 examples of the PR-anaphora schema.
Out of these, 77,118 (0.148\%) were found to be sheaf-contextual; 36,938,948 (71.1\%) were found to be CbD-contextual.
Here sheaf-contextual means that the model has $\SF < 1/6$ and CbD-contextual means that the model has $\Delta < 2$.

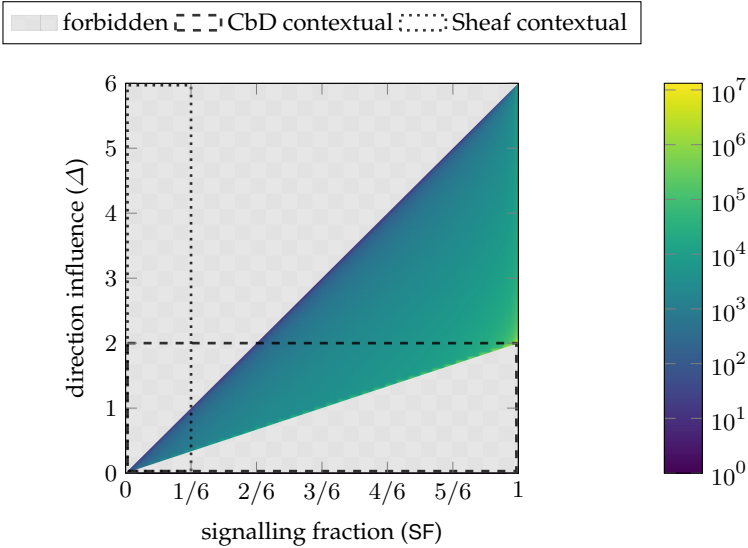
\begin{figure}[h]
    \centering
    \begin{tikzpicture} 
    \begin{axis}[
    xmin=0,xmax=1,ymin=0,ymax=6, 
    axis on top,
    width=0.50\textwidth,
    height=0.50\textwidth,
    colorbar,
    colorbar style={
            ytick={0, 1, 2, 3, 4, 5, 6, 7},
            yticklabel={$10^{\pgfmathprintnumber{\tick}}$},
        },
    colormap name=viridis,
    xtick={0, 0.1666, 0.3333, 0.5, 0.6666, 0.8333, 1},
    xticklabels={$0$, $1/6$, $2/6$, $3/6$, $4/6$, $5/6$, $1$},
    ytick={0, 1, 2, 3, 4, 5, 6},
    legend style={
            area legend,
            anchor=south,
            at={(0.5,1.1)},
            legend columns=-1,
            legend cell align = {left},
    },
    xlabel={signalling fraction ($\SF$)},
    ylabel={direction influence ($\Delta$)},
]
\tikzset{forbidden/.style={pattern=checkerboard light gray, pattern color=gray, draw=gray, opacity=0.1}}
\tikzset{contextual/.style={opacity=0.8, line width=1.0pt}}

    \addplot [draw=none,point meta=x, forget plot] coordinates {(0,0) (7.1219,0)};

    \addplot [forget plot] graphics [xmin=0,xmax=1,ymin=0,ymax=6] {
        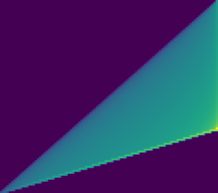
    };
    \addplot [fill=white, forget plot] coordinates {(0,0) (0,6) (1, 6) };
    \addplot [forbidden, forget plot] coordinates {(0,0) (0,6) (1, 6) };
    \addplot [fill=white, forget plot] coordinates {(0,0) (1, 0) (1, 2)};
    \addplot [forbidden] coordinates {(0,0) (1, 0) (1, 2)};

    \pgfmathsetmacro{\xoffset}{0.005}
    \pgfmathsetmacro{\yoffset}{\xoffset*6}
\addplot+[contextual, draw=black, dashed, no markers] plot coordinates {(\xoffset, 2) (1-\xoffset, 2) (1-\xoffset,\yoffset) (\xoffset,\yoffset) (\xoffset, 2)};
\addplot+[contextual, draw=black, dotted, no markers] plot coordinates {(1/6, \yoffset) (1/6, 6-\yoffset) (\xoffset, 6-\yoffset) (\xoffset,\yoffset) (1/6, \yoffset)};

\addlegendentry{forbidden}
\addlegendentry{CbD contextual}
\addlegendentry{Sheaf contextual}

\end{axis}
\end{tikzpicture}
\caption{The distribution of the instances in the space of direct influence and signalling fraction, which is equally divided into 200 times 200 bins. The colour of each bin represents the log of the number of instances that fall into that bin. As determined by Equation \ref{eq:inequality}, certain regions of the space are not accessible to the instances, which is shown as \emph{forbidden} in the figure. The regions where the instances are either CbD contextual or sheaf contextual are outlined in the figure. }
\label{fig:sfdeltaheatmap}
\end{figure}
Figure \ref{fig:histograms} shows the distribution of signalling fraction $\SF$ and direct influence $\Delta$ of the examples.
The distribution of signalling fraction $\SF$ can be seen heavily skewed towards 1 and sharply peaking at 1, while 
the distribution of direct influence $\Delta$ sharply peaks at 2. 
Our hypothesis is that in the PR-anaphora schema examples, BERT often predicts the same word for the masked token in all the contexts with high probability,  resulting all the $\epsilon$ values  to be either close to -1 or 1.
In order to see why, suppose that $(\epsilon_1, \epsilon_2, \epsilon_3) = (1, 1, 1)$ or $(-1, -1, -1)$; this would result in $\SF = \max_i|\epsilon_i| = 1$ and $\Delta = |\epsilon_1 - \epsilon_2| + |\epsilon_2 - \epsilon_3| + |\epsilon_3 + \epsilon_1| = 2$. 
Figure \ref{fig:sfdeltaheatmap} shows the distribution of the examples in the space of $\Delta$ and $\SF$. The majority of the examples are concentrated at the point $(\SF, \Delta) = (1, 2)$, which is the point where the $\epsilon$ values are all equal to 1 or -1.  
As the sheaf-contextual region is fully enclosed by the CbD-contextual region, it follows that the sheaf-contextual models are a strict subset of the CbD-contextual models.
Table \ref{tab:instanceswithsfdelta} shows samples of the instances of the PR-anaphora schema with their signalling fraction $\SF$ and direct influence $\Delta$.

\begin{table}[h]
\footnotesize
    \centering
    \begin{tabular}{ll|lll|cc}
        \toprule
        \multicolumn{2}{c|}{nouns} & \multicolumn{3}{c|}{adjectives} & $\SF$ & $\Delta$ \\
        \midrule
scholar & opposition & British & great & Russian & 0.963 & 2.217 \\
apprentice & side & last & former & new & 0.992 & 2.626 \\
camera & telescope & special & main & old & 0.824 & 2.737 \\
tea & tree & Japanese & small & Australian & 0.787 & 3.260 \\
prey & day & regular & important & primary & 1.000 & 2.035 \\
\midrule
fox & passenger & last & American & female & 0.748 & \bf{1.496} \\
dwarf & bass & new & Russian & black & 0.985 & \bf{1.979} \\
dancer & scene & main & nude & German & 0.969 & \bf{1.938} \\
videos & text & short & full & sexual & 1.000 & \bf{2.000} \\
series & sheep & regular & single & famous & 1.000 & \bf{1.999} \\
\midrule
fire & relief & American & direct & poor & \bf{0.165} & \bf{0.376} \\
person & photographer & British & American & French & \bf{0.155} & \bf{0.310} \\
track & rule & American & new & British & \bf{0.165} & \bf{0.779} \\
memory & saint & great & important & certain & \bf{0.141} & \bf{0.283} \\
island & architect & new & British & Japanese & \bf{0.142} & \bf{0.292} \\
\bottomrule
    \end{tabular}
    \caption{Randomly selected samples of instances of the PR-anaphora schema with signalling fraction  and direct influence values,  highlighted when  contextual.}
    \label{tab:instanceswithsfdelta}
\end{table}

\subsection{Similar Noun Subset of Dataset}
The examples of the schema restricted the set of adjectives $\{X_1, X_2, X_3\}$  to the 5 most  frequent adjectives of each noun pair $(O_1. O_2)$. It however did not impose any restrictions on the noun pairs themselves.  As a result, we come across noun pairs that are very unlikely to have occurred together in the same context.  Some of these noun pairs even lead to contextual examples, for instance  the pair \emph{(memory, saint)},  from Table \ref{tab:instanceswithsfdelta}, which is  both CbD and sheaf contextual.  Such pairs of nouns can still share adjectives; as one can see both \emph{memory} and \emph{saint} are commonly modified by any of the three adjectives  \emph{great, important, certain}. In order to  filter out these instances, we only consider pairs of nouns that are highly semantically related, and as a result have often occurred together in the same context.  Cosine of the angle between two word vectors has proven to be a good measure of semantic similarity. We form a similar noun subset of the dataset by restricting it to  with top $1\%$ most semantically similar noun pairs.  A selection of these and their degrees of contextuality are presented in Table \ref{tab:instanceswithsfdelta}.  This resulted in an increase in the  percentage of  the contextual instances.  The  percentage of the sheaf-contextual examples increased to 0.50\% from 0.0148\% and that  of the CbD-contextual ones increased to 81.83\% from 71.1\%.  See Figure \ref{fig:histograms-similar} for the  histograms of direct influence and SF. 

\begin{figure}[t!]
    \footnotesize
        \centering
    \begin{subfigure}[b]{0.45\textwidth}
        \begin{tikzpicture}

\def\coordinatesone{
(0.00000, 8)
(0.01852, 36)
(0.03704, 97)
(0.05556, 211)
(0.07407, 299)
(0.09259, 434)
(0.11111, 575)
(0.12963, 774)
(0.14815, 973)
(0.16667, 973)
}
\def\coordinatestwo{
(0.16667, 1133)
(0.18367, 1277)
(0.20068, 1494)
(0.21769, 1720)
(0.23469, 1926)
(0.25170, 2107)
(0.26871, 2167)
(0.28571, 2349)
(0.30272, 2723)
(0.31973, 2854)
(0.33673, 3201)
(0.35374, 3367)
(0.37075, 3718)
(0.38776, 3869)
(0.40476, 4182)
(0.42177, 4262)
(0.43878, 4690)
(0.45578, 4866)
(0.47279, 4995)
(0.48980, 5390)
(0.50680, 5571)
(0.52381, 5811)
(0.54082, 5970)
(0.55782, 6364)
(0.57483, 6448)
(0.59184, 7068)
(0.60884, 7328)
(0.62585, 7750)
(0.64286, 7993)
(0.65986, 8418)
(0.67687, 8498)
(0.69388, 8822)
(0.71088, 9356)
(0.72789, 9843)
(0.74490, 10282)
(0.76190, 11057)
(0.77891, 11713)
(0.79592, 11863)
(0.81293, 13038)
(0.82993, 13940)
(0.84694, 15184)
(0.86395, 16148)
(0.88095, 17258)
(0.89796, 20273)
(0.91497, 22033)
(0.93197, 26185)
(0.94898, 32057)
(0.96599, 42310)
(0.98299, 85322)
(1.00000, 85322)
}

\begin{semilogyaxis}[
    xmin=0, xmax=1,
    xtick={0,1/6,2/6,3/6,4/6,5/6,1},
    xticklabels={$0$,  $1 / 6$,  $2 / 6$,  $3 / 6$,  $4 / 6$,  $5 / 6$,  $1$},
    grid,
    grid style=dashed,
    legend pos={outer north east},
    legend cell align={left},
    ytick={1,10,100,1000,10000,100000,1000000,10000000},
    width=7cm,
    ]
\addplot+[ybar interval,mark=no,color=gray,fill=gray!50,ybar legend] plot coordinates {
\coordinatestwo
};
\addplot+[ybar interval,mark=no,color=black,fill=black!50,ybar legend] plot coordinates {
\coordinatesone
};
\end{semilogyaxis}
\end{tikzpicture}
        \caption{Signalling fraction}
        \label{fig:sfhistosim}
    \end{subfigure}
    \hfill
    \begin{subfigure}[b]{0.45\textwidth}
        \begin{tikzpicture}

\def\coordinatesone{
(0.00000, 88)
(0.10526, 530)
(0.21053, 1380)
(0.31579, 2661)
(0.42105, 4146)
(0.52632, 5349)
(0.63158, 7237)
(0.73684, 9360)
(0.84211, 11239)
(0.94737, 13162)
(1.05263, 15483)
(1.15789, 17969)
(1.26316, 20810)
(1.36842, 23642)
(1.47368, 27533)
(1.57895, 32817)
(1.68421, 41245)
(1.78947, 56898)
(1.89474, 135874)
(2.00000, 135874)
}
\def\coordinatestwo{
(2.00000, 30185)
(2.10256, 12754)
(2.20513, 8883)
(2.30769, 6919)
(2.41026, 5379)
(2.51282, 4284)
(2.61538, 3532)
(2.71795, 3109)
(2.82051, 2444)
(2.92308, 2257)
(3.02564, 1815)
(3.12821, 1617)
(3.23077, 1371)
(3.33333, 1185)
(3.43590, 1000)
(3.53846, 885)
(3.64103, 696)
(3.74359, 602)
(3.84615, 513)
(3.94872, 513)
(4.05128, 382)
(4.15385, 333)
(4.25641, 295)
(4.35897, 223)
(4.46154, 171)
(4.56410, 154)
(4.66667, 145)
(4.76923, 115)
(4.87179, 102)
(4.97436, 66)
(5.07692, 62)
(5.17949, 50)
(5.28205, 48)
(5.38462, 36)
(5.48718, 22)
(5.58974, 15)
(5.69231, 8)
(5.79487, 7)
(5.89744, 0)
(6.00000, 0)
}

\begin{semilogyaxis}[
    xmin=0, xmax=6,
    xtick={0,1,2,3,4,5,6},
    xticklabels={$0$, $1$, $2$, $3$, $4$, $5$, $6$},
    grid,
    grid style=dashed,
    legend pos={outer north east},
    legend cell align={left},
    ytick={1,10,100,1000,10000,100000,1000000,10000000},
    width=7cm,
    ]
\addplot+[ybar interval,mark=no,color=gray,fill=gray!50,ybar legend] plot coordinates {
\coordinatestwo
};
\addplot+[ybar interval,mark=no,color=black,fill=black!50,ybar legend] plot coordinates {
\coordinatesone
};
\end{semilogyaxis}
\end{tikzpicture}
        \caption{Direct influence}
        \label{fig:deltahistosim}
    \end{subfigure}
    \caption{Histograms of (a) the signalling fraction and (b) the direct influence of the 519,660 models constructed for the similar nouns subset of the dataset.}
    \label{fig:histograms-similar}
\end{figure}
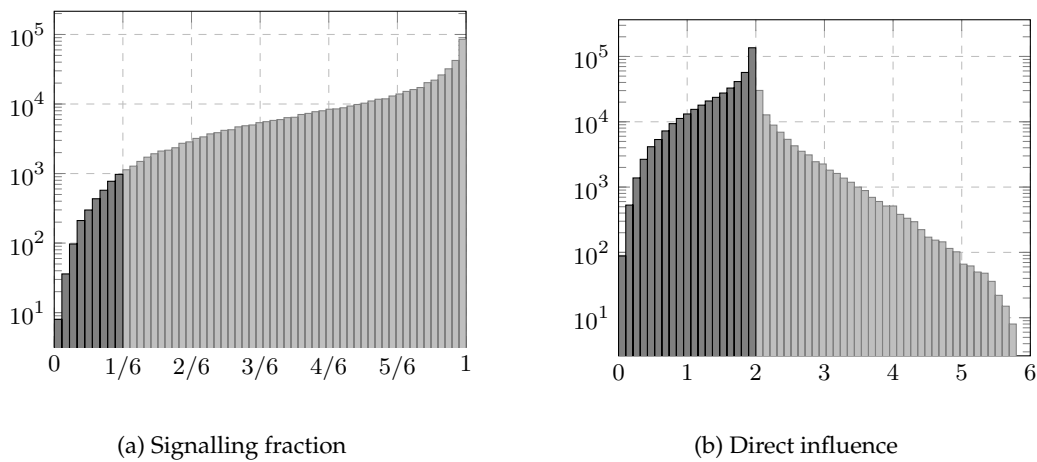

\begin{table}[b!]
    \footnotesize
    \centering
    \begin{tabular}{ll|lll|cc}
        \toprule
        \multicolumn{2}{c|}{nouns} & \multicolumn{3}{c|}{adjectives} & $\SF$ & $\Delta$ \\
        \midrule
 television & tv & nationwide & web & live & 0.74 & \textbf{1.47} \\
grandmother & grandfather & paternal & great & maternal & 0.37 & \textbf{0.73} \\
painting & sculpture & modern & great & famous & 0.86 & 2.09 \\
artist & facility & medical & new & national & 1.00 & 5.25 \\
supplier & producer & local & single & main & 0.73 & \textbf{1.62} \\

\midrule
railroad & railway & national & new & main & \textbf{0.11} & \textbf{0.22} \\
journalist & reporter & black & Italian & American & \textbf{0.12} & \textbf{0.24} \\
 station & hospital & small & main & large & \textbf{0.14} & \textbf{0.54} \\
 creature & snake & common & giant & wooden & \textbf{0.13} & \textbf{0.26} \\
assassin & journalist & Japanese & American & French & 0.58 & \textbf{1.15} \\

\bottomrule
    \end{tabular}
    \caption{A selection of  most similar noun pairs and their adjectives.}
    \label{tab:instanceswithsfdelta}
\end{table}

\section{Analysis of the Results}

BERT (Bidirectional Encoder Representations from Transformers) \cite{Devlin2019} is a language encoder that is based on the Transformer architecture \cite{Vaswani2017}.
Given a sequence of tokens $(x_1, x_2, \dots, x_n)$, BERT encodes each token with a vector, resulting in a sequence of embedding vectors $(\mathbf{x}_1, \mathbf{x}_2, \dots, \mathbf{x}_n)$. 
The embedding vectors are then fed into an transformer encoder which is a stack of multi-head self-attention layers to produce a sequence of  embeddings vectors $(\mathbf{y}_1, \mathbf{y}_2, \dots, \mathbf{y}_n)$. The self-attention layers allow information to flow between any two positions in the input sequence, thereby modifying the embedding vectors to capture the context of the input sequence.
Thus the embedding vectors are considered to be  \emph{contextualised}, rather than \emph{static} as in word2vec \cite{mikolov2013}.
See Figure \ref{fig:bertflow} for a high level overview of the BERT architecture. In this section we present a geometric interpretation of the predictions of BERT, so that we can relate the factors involved in these predictions to the parameters that affect  contextuality.

\subsection{BERT logit score and the $\epsilon$ parameter of empirical tables}
One of the two tasks that BERT was trained on was masked language modelling. 
In this task,  a fraction of the input tokens are masked and the model is trained to predict the masked tokens. For this purpose, a further feedforward layer was added on top of the stack of self-attention layers to produce a sequence of output vectors $(\mathbf{p}_1, \mathbf{p}_2, \dots, \mathbf{p}_n)$, one for each token in the input sequence.
Suppose that the i-th token is masked. To obtain the predicted distribution of tokens on the i-th token, the corresponding output vector $\mathbf{p}_i$ is compared against the embedding vector of all candidate tokens and a \emph{logit} score is produced for each token.
A softmax function is then applied to the logit scores to obtain the probability distribution over the vocabulary.
More precisely, the logit score $l_j$ and probabilty $P_j$ of the j-th candidate token are given by:

\noindent\begin{minipage}{0.4\textwidth}
\begin{align}
    l_j = \mathbf{p}_i \cdot \mathbf{e}_j + b_j
    \label{eq:logit}
\end{align}
    \end{minipage}%
    \begin{minipage}{0.2\textwidth}\centering
    and 
    \end{minipage}%
    \begin{minipage}{0.4\textwidth}
\begin{eqnarray}
    P_j &=& \frac{\exp{l_j}}{\sum_{k=1}^{|V|} \exp{l_k}},
\label{eq:softmax}
\end{eqnarray}
\end{minipage}\vskip1em
\noindent where $b_j$ is a token-specific bias, $\mathbf{e}_j$ is the embedding vector of the j-th candidate token, and $\mathbf{p}_i$ is the output vector of the masked i-th token.

In the above, $|V|$ is the size of the vocabulary.
In our case, the vocabulary comprises of our  two outcomes, i.e.\ the two nouns in the PR-anaphora schema. Using equations \ref{eq:logit} and \ref{eq:softmax}, below in Prop. \ref{prop:main} we  prove a result which connects the BERT logit scores to the empirical table of the PR-like model descirbing the PR-anaphora schema:

\vspace{0.4cm}
\begin{tabular}{r|ccccc}
    & $(O_1, O_1)$ & $(O_1, O_2)$ & $(O_2, O_1)$ & $(O_2, O_2)$   \\ \hline
    $(X_1, X_2)$ & $\mathbf{p}_{X_1} \cdot \mathbf{e}_{O_1} + b_{O_1}$ & 0 & 0 & $\mathbf{p}_{X_1} \cdot \mathbf{e}_{O_2} + b_{O_2}$  \\
    $(X_2, X_3)$ & $\mathbf{p}_{X_2} \cdot \mathbf{e}_{O_1} + b_{O_1}$ & 0 & 0 & $\mathbf{p}_{X_2} \cdot \mathbf{e}_{O_2} + b_{O_2}$  \\
    $(X_3, X_1)$ & 0 & $\mathbf{p}_{X_3} \cdot \mathbf{e}_{O_1} + b_{O_1}$ & $\mathbf{p}_{X_3} \cdot \mathbf{e}_{O_2} + b_{O_2}$ & 0
\end{tabular}

\vspace{0.4cm}
\noindent
Here,  the logit scores are shown instead of probabilities for clarity. The probabilities are obtained by feeding the logit scores into the softmax function per row.

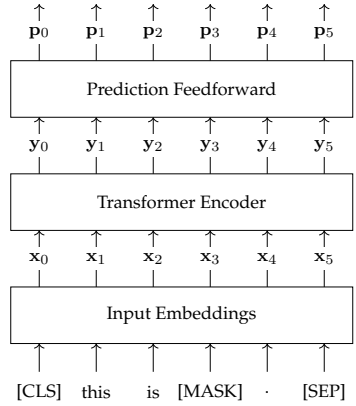
\begin{figure}[t!]
    \centering
    \scalebox{.75}
{\begin{tikzpicture}[node distance=1cm and 0.5cm, auto]
    \def\arrowlen{1}
    \def\boxheight{1.0}
    \def\boxwidth{6}
    \def\xmid{2.5}
    \foreach \value/\text [count=\i from 0] in {0/{[CLS]},1/this,2/is,3/{[MASK]},4/.,5/{[SEP]}} {
        \node[below, minimum height=0.7cm] at (\i, \arrowlen/2) {\text};
        \draw[->] (\i, \arrowlen/2) -- (\i, \arrowlen);
        \draw[->] (\i, \arrowlen+\boxheight) -- (\i, 2*\arrowlen+\boxheight);
        \draw[->] (\i, 2*\arrowlen+2*\boxheight) -- (\i, 3*\arrowlen+2*\boxheight);
        \draw[->] (\i, 3*\arrowlen+3*\boxheight) -- (\i, 4*\arrowlen+3*\boxheight);
        \node[fill=white, inner sep=0.1cm] at (\i, 1.5*\arrowlen+\boxheight) {\(\mathbf{x}_{\value}\)};
        \node[fill=white, inner sep=0.1cm] at (\i, 2.5*\arrowlen+2*\boxheight) {\(\mathbf{y}_{\value}\)};
        \node[fill=white, inner sep=0.1cm] at (\i, 3.5*\arrowlen+3*\boxheight) {\(\mathbf{p}_{\value}\)};
    }
    
    \node[draw, minimum width=\boxwidth cm, minimum height=\boxheight cm] at (\xmid, \arrowlen+\boxheight/2) {Input Embeddings};
    \node[draw, minimum width=\boxwidth cm, minimum height=\boxheight cm] at (\xmid, 2*\arrowlen+3*\boxheight/2) {Transformer Encoder};
    \node[draw, minimum width=\boxwidth cm, minimum height=\boxheight cm] at (\xmid, 3*\arrowlen+5*\boxheight/2) {Prediction Feedforward};
\end{tikzpicture}}
\caption{A flow chart illustrating how the embedding vectors are transformed into the output vectors in a BERT model. Extra tokens [CLS] and [SEP] are added to the input sequence to indicate the start and end of the sequence, while the [MASK] token is used to indicate the mask.}
\label{fig:bertflow}
\end{figure}

\begin{figure}[b!]
    \centering
\begin{tikzpicture}[scale=1]

    \def\xmin{-1}
    \def\xmax{3}
    \def\ymin{-1}
    \def\ymax{3}
    \def\kx{1}
    \def\ky{2}
    \def\b{2}

    \draw[->] (\xmin,0) -- (\xmax,0) node[right] {};
    \draw[->] (0,\ymin) -- (0,\ymax) node[above] {};
    

    \coordinate (lineleft) at (\xmin,{-\kx/\ky*\xmin + \b/\ky});
    \coordinate (lineright) at (\xmax,{-\kx/\ky*\xmax + \b/\ky});
    \draw[thick] (lineleft) -- (lineright) node[right] {$\mathbf{p} \cdot \Delta \mathbf{x} + \Delta b = 0$};

    \pgfmathsetmacro{\normk}{sqrt(\kx*\kx + \ky*\ky)}
    \pgfmathsetmacro{\d}{abs(\b) / \normk}
    \pgfmathsetmacro{\dx}{\d * \kx / \normk}
    \pgfmathsetmacro{\dy}{\d * \ky / \normk}

    \coordinate (p1) at (165:1.5);
    \coordinate (p2) at (25:4);
    \coordinate (p3) at (50:2);
    
    
    \draw[->,thick] (0,0) -- (p1) node[anchor=south east] {$\mathbf{p}_1$};
    \draw[->,thick, red] (0,0) -- (p2) node[anchor=south east] {$\mathbf{p}_2$};
    \draw[->,thick] (0,0) -- (p3) node[anchor=south east] {$\mathbf{p}_3$};
    
    \draw[dashed] (p1) -- ($(lineleft)!(p1)!(lineright)$) node[midway, right] {$\frac{\Delta l_1}{\| \Delta \mathbf{x} \|}$};
    \draw[dashed] (p2) -- ($(lineleft)!(p2)!(lineright)$) node[midway, right] {$\frac{\Delta l_2}{\| \Delta \mathbf{x} \|}$};
    \draw[dashed] (p3) -- ($(lineleft)!(p3)!(lineright)$) node[midway, right] {$\frac{\Delta l_3}{\| \Delta \mathbf{x} \|}$};
\end{tikzpicture}
\caption{A 2-dimensional sketch of a geometric interpretation of the mask predictions from BERT for the PR-anaphora schema. The vectors $\mathbf{p}_i$ are the output vectors of the masked token for the $i$-th context in the schema.
The distance from a predictor vector $\mathbf{p}_i$ to the hyperplane defined by the equation $\mathbf{p} \cdot \Delta \mathbf{x} + \Delta b = 0$ coincides with $\Delta l_i / \|\Delta \mathbf{x}\|$. 
As $\epsilon_i$ relates to $\Delta l_i$ monotonically, specifically $\epsilon_i = \tanh(\Delta l_i / 2)$, the signalling fraction $\SF = \max |\epsilon_i|$ depends only on the prediction vectors furthest away from the hyperplane. In the figure, the prediction vector $\mathbf{p}_2$ (coloured red) is the furthest away from the hyperplane.
}
\label{fig:hyperplane}
\end{figure}
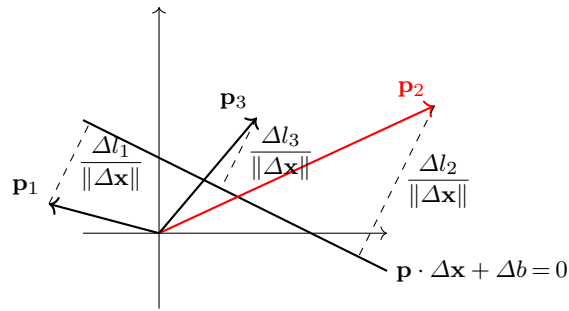

\begin{Proposition}\label{prop:main}
   The logit scores of the masked token given by BERT relates to the $\epsilon$ parametrisation of the PR-like model as follows: 
    \begin{align}
        \epsilon = \tanh\left(\frac{1}{2}(\mathbf{p} \cdot \Delta \mathbf{x} + \Delta b)\right).
        \label{eq:epsilon}
    \end{align}
where $\mathbf{p}$ is the output vector of the masked token, $\Delta \mathbf{x} = \mathbf{e}_{O_1} - \mathbf{e}_{O_2}$ is the difference between the embedding vectors of the two nouns $O_1$ and $O_2$, and $\Delta b = b_{O_1} - b_{O_2}$ is the difference between the bias terms of the two nouns in the masked modelling prediction head of BERT. 
\end{Proposition}
\begin{proof}
    Recall the logit scores for the two outcomes are given by: 
    \begin{align}
        l_{O_1} = \mathbf{p} \cdot \mathbf{e}_{O_1} + b_{O_1} \text{\quad and \quad}
        l_{O_2} = \mathbf{p} \cdot \mathbf{e}_{O_2} + b_{O_2}.
    \end{align}
    Since the probabilities are obtained by taking the softmax of the logit scores, the ratio of the probabilities is given by:
    \begin{align}
        \frac{P_{O_1}}{P_{O_2}} = \frac{e^{l_{O_1}}}{e^{l_{O_2}}} = e^{l_{O_1} - l_{O_2}}.
    \end{align}
    By definition, $P_{O_1} = (1+\epsilon)/2$ and $P_{O_2} = (1-\epsilon)/2$. Thus we have:
    \begin{align}
        \log \frac{1+\epsilon}{1-\epsilon} = l_{O_1} - l_{O_2} = \mathbf{p} \cdot \Delta \mathbf{x} + \Delta b.
    \end{align}
    Further using the fact that $\tanh^{-1}(x) = \frac{1}{2} \log \frac{1+x}{1-x}$, we obtain the desired result.
\end{proof}

Note that the $\tanh$ function is montonically increasing. 
Therefore we can use the difference in logit scores given below as a proxy for the the value of $\epsilon$. 
\begin{equation}
    \Delta l := \mathbf{p} \cdot \Delta \mathbf{x} + \Delta b 
    \label{eq:deltal}
\end{equation}
The value $\Delta l$ can be interpreted as $\|\Delta \mathbf{x}\|$ times the distance from the prediction vector $\mathbf{p}$ to the hyperplane defined by the equation $\mathbf{p} \cdot \Delta \mathbf{x} + \Delta b = 0$. 
\noindent
A visualisation of this geometric interpretation is shown in Figure \ref{fig:hyperplane}.
Assuming an isotropic distribution of the prediction vectors $\mathbf{p}$, Equation \ref{eq:deltal} suggests that the value of $\Delta l$ is directly proportional to the Euclidean distance between the embedding vectors of the two nouns $\| \Delta \mathbf{x} \|$, which in turn non-linearly scales the value of $\epsilon$ through Equation \ref{eq:epsilon}.
Since a higher $\epsilon$ value implies less contextuality in both the sheaf and CbD frameworks, we expect that the value of $\| \Delta \mathbf{x} \|$ plays an important role in determining whether a model is contextual.
%
%
Equation \ref{eq:deltal} can be thought as a hyperplane in the word embedding space of BERT that allows a geometric interpretation of the predictions of BERT for the PR-anaphora schema which is shown in Figure \ref{fig:hyperplane}.
The bias difference $\Delta b$ serves to offset the hyperplane from the origin, while the difference in embedding vectors $\Delta \mathbf{x}$ determines the orientation of the hyperplane. 

\subsection{Factors affecting contextuality}
 \label{eq:deltal}
\begin{table}[t!]
    \centering
    \begin{tabular}{l|rrr|rrr}
    \toprule
        & \multicolumn{3}{c|}{sheaf} & \multicolumn{3}{c}{CbD} \\
    feautre & Kendall & Spearman & Pearson & Kendall & Spearman & Pearson \\
    \midrule
		nouns\_entropy & -0.0202 & -0.0302 & -0.0224 & -0.0125 & -0.0188 & -0.0140 \\
		adjectives\_entropy & -0.0214 & -0.0322 & -0.0257 & -0.0132 & -0.0198 & -0.0167 \\
        \midrule
		bert\_euclidean\_dist & \bf{0.0587} & \bf{0.0877} & \bf{0.0809} & \bf{0.0450} & \bf{0.0674} &\bf{ 0.0590} \\
		bert\_bias\_diff & 0.0334 & 0.0500 & 0.0123 & 0.0115 & 0.0173 & -0.0054 \\
        \bottomrule
    \end{tabular}
    \caption{Linear correlation coefficients between the features and the contextuality of the instances of the PR-anaphora schema, in the full dataset. }
    \label{tab:correlation}
\end{table}

\begin{table}[b!]
    \centering
    \begin{tabular}{lrr|rr}
    \toprule
        & \multicolumn{2}{c|}{signalling fraction} & \multicolumn{2}{c}{Delta} \\
    feautre & linear & cubic & linear & cubic \\
    \midrule
        nouns_entropy & 0.0005 & 0.0006 & 0.0002 & 0.0002 \\
        adjectives_entropy & 0.0007 & 0.0007 & 0.0003 & 0.0004 \\
        \midrule
        bert_euclidean_dist & \bf{0.0065} & \bf{0.0092} &\bf{0.0035} & \bf{0.0050} \\
        bert_bias_diff & 0.0001 & 0.0003 & 0.0000 & 0.0001 \\
        \bottomrule
    \end{tabular}
    \caption{Comparison of the $R^2$ values of the linear and cubic regression models, in the full dataset.}
    \label{tab:regression}
\end{table}

In the previous section, we showed that the differences in  BERT's logit scores, i.e. $\Delta l$ can be used as a proxy for the values of $\epsilon$, which are used to compute the entries of  the empirical tables. Further, in Equation \ref{eq:deltal} we showed that $\Delta l$ is directly proportional to Euclidean distance between the vectors of the two nouns in the instances of the PR-Anaphora schema.  Although this finding relates Euclidean distance to contextuality, it does not rule out other features of either the vectors of the nouns or the nouns themselves that might affect it too.  For instance, Equation \ref{eq:deltal} also hosts the variable $\Delta b$, which is the difference between the bias terms of the vectors of the two nouns.  In this section, we are interested in finding out which one of these two is most correlated with contextuality.   Specifically, we compute the degree of correlation between three features of BERT's predicted noun vectors, as well as two other independent features of the nouns. Our goal is to investigate  which one of these features correlate best with contextuality.  For correlation, we compute Spearman, Pearson and Kendall degrees. For BERT features, we consider Euclidean distances and the difference between the biases of   vectors.    In order to compute these distances, we use the pre-trained BERT model \texttt{bert-base-uncased} provided by the HuggingFace Transformers library~\cite{huggingfacetransformers}. 


Computing the differences between  features of word vectors are not the only ways of measuring and  comparing the statistical information encoded in them. In fact, the general rule governing BERT is that it  chooses the word that has occurred most in the corpus. This is too rough of a feature to be used in our schema instances, since we need BERT to choose between the words with equal probability.  Here entropy can come to help.  Entropy  is an often used method when it comes to computing  the imbalances between word frequencies.  If two words have similar frequencies,  entropy will peak. On the other hand, if one has a low and the other a high  frequency, we will have a low entropy.    In order to find out whether entropy is related to contextuality, we compute degrees of correlation between both of our contextuality measures and the entropy of nouns and  adjectives.

\begin{table}[t!]
    \centering
    \begin{tabular}{l|rrr|rrr}
    \toprule
        & \multicolumn{3}{c|}{sheaf (SF < 1/6)} & \multicolumn{3}{c}{CbD (Delta < 2)} \\
    feautre & Kendall & Spearman & Pearson & Kendall & Spearman & Pearson \\
    \midrule
		nouns\_entropy & -0.0436 & -0.0654 & -0.0625 & -0.0394 & -0.0592 & -0.0541 \\
		adjectives\_entropy & -0.0227 & -0.0341 & -0.0362 & -0.0193 & -0.0288 & -0.0246 \\
        \midrule
		bert\_euclidean\_dist & \bf{0.1234} & \bf{0.1837} & \bf{0.1963} & \bf{0.1155} & \bf{0.1722} & \bf{0.1787} \\
		bert\_bias\_diff & -0.0821 & -0.1241 & -0.0996 & -0.0670 & -0.1005 & -0.0679 \\
        \bottomrule
    \end{tabular}
    \caption{Correlation coefficients between the features and the contextuality of the instances of the PR-anaphora schema, in the similar noun subset of the dataset.}
    \label{tab:correlationsim}
\end{table}

\begin{table}[b!]
    \centering
\begin{tabular}{l|rr|rr}
\toprule
& \multicolumn{2}{c|}{signalling fraction} & \multicolumn{2}{c}{Delta} \\
feautre & linear & cubic & linear & cubic \\
\midrule
nouns_entropy & 0.0001 & 0.0003 & 0.0003 & 0.0004\\
adjectives_entropy & 0.0012 & 0.0018 & 0.0024 & 0.0027\\
\midrule
bert_euclidean_dist & \textbf{0.0779} & \textbf{0.0803} & \textbf{0.0573} & \textbf{0.0581}\\
bert_bias_diff & 0.0036 & 0.0109 & 0.0020 & 0.0059\\
\bottomrule
\end{tabular}
    \caption{Comparison of the $R^2$ values of the linear and cubic regression models,  in the similar noun subset of the dataset.}
    \label{tab:regressionsim}
\end{table}

Let us first consider the full dataset. In this dataset, the  correlation scores  between each of the above features and both of our contexutaliy measures, i.e. CbD's delta and sheaf theory's SF,  are shown in Table \ref{tab:correlation}. The highest correlations for any of the correlation scores are with Euclidean distance. BERT's bias differences provided the second best set of correlations, although they were much lower that Euclidean distance. Finally, entropies resulted in negative correlation for all of the correlation measures and both SF and delta. 

\begin{figure}[t!]
    \centering
    \includegraphics[width=0.49\textwidth]{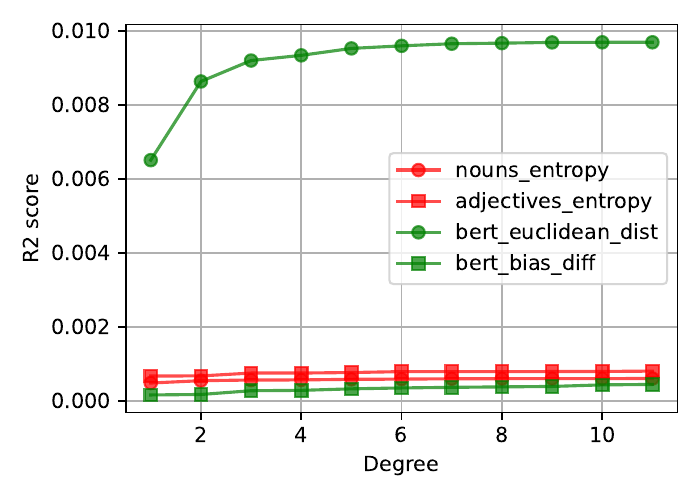}
    \includegraphics[width=0.49\textwidth]{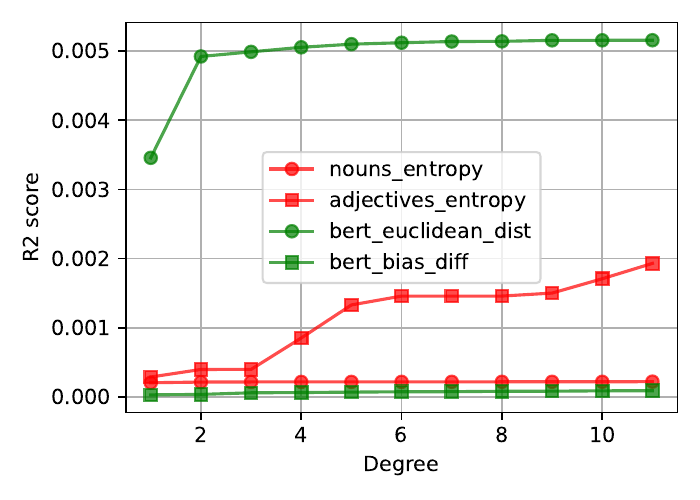}
    \caption{The $R^2$ scores of the polynomial regression models at different polynomail degrees predicting (left) the signalling fraction and (right) the direct influence.}
    \label{LOST}
\end{figure}

The above correlations are all  statistically significant (p-values are all below 0.01),  but on the  low side (below 0.005).   This indicates the presence of a non linear correlation. To test this, we  trained a polynomial logistic regression model, on a range of degrees from 2 to 10. We chose the cubic degree polynomial as a cut off point.  The $R^2$ values for linear vs cubic correlations  are shown in  Table \ref{tab:regression}.  The results for all the 10 degrees are plotted in Figure \ref{LOST}. Clearly, there is   a 2-3 times increase in the correlations of the cubic models in comparison to the linear one. Again, the highest correlation was with Euclidean distance.  Since $R^2$ is the square of Pearson's correlation, all the values are positive. Naturally, with the $R^2$, the entropies provided better correlations than BERT bias differences, but both of these were still quite low. This shows that geometric distances play an important role when it comes to predicting contextuality. Understanding this fact requires a more in depth study of both settings.

The similar noun subset led to very similar results, see  Table \ref{tab:correlationsim}. Here, again we observe  (1) an increase in the cubic regression correlations in comparison to the linear ones, and (2) Euclidean distance provides the  highest correlation  with both SF and delta. These provide further experimental evidence that   the Euclidean distances between BERT's word vectors  are the best statistical predictors of  degrees of  contextuality.

\section{Discussion and Conclusion}

\begin{figure}[t!]
    \includegraphics[width=0.5\textwidth]{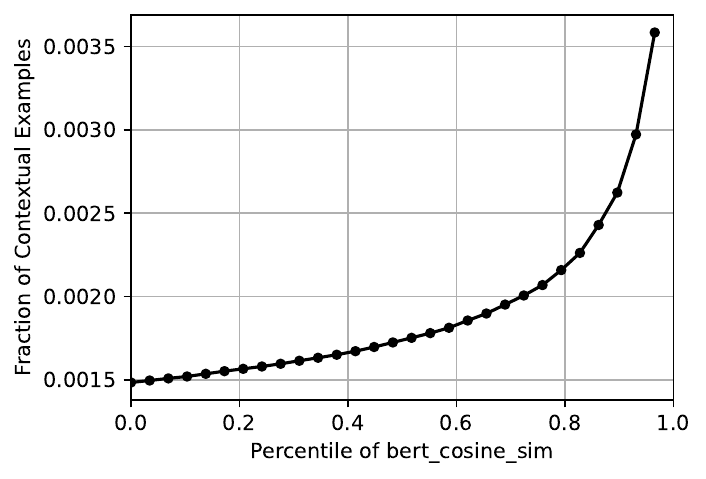}
    \includegraphics[width=0.5\textwidth]{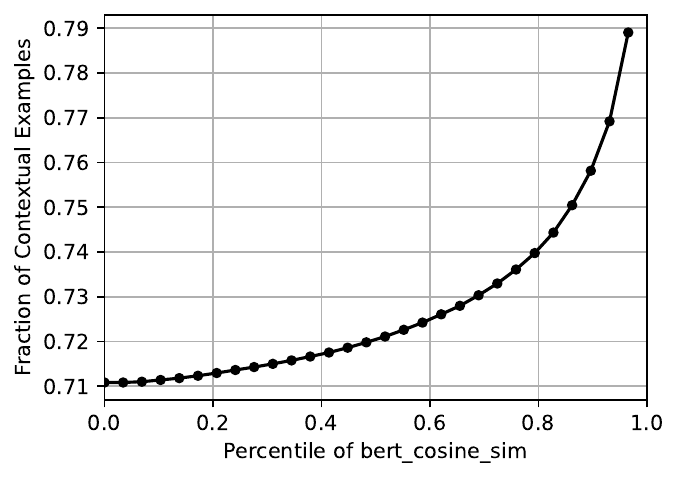}
    \caption{
        The fraction of sheaf-contextual (left) and CbD-contextual (right) instances at different subsets of the dataset created by considering the most similar noun pairs at different percentile thresholds.
    }
    \label{fig:simthreshold}
\end{figure}

In this paper, we set to find out whether quantum contextuality can occur in natural language. We built a linguistic schema and modelled it over a quantum contextual scenario. We then instantiated this schema using the available snapshot of  the Simple English Wikipedia. Probability distributions of the instances were collected using the masked word prediction capability of the large language model BERT. Since natural language data is signalling, one should work in more general frameworks, such as the Contextuality-by-Default (CbD)  and the signalling corrected edition of the sheaf theoretic model of contextuality. Computing degrees of contextuality in either of these frameworks led to the discovery of many CbD and sheaf theoretic contextual instances. In order to investigate the reason behind this discovery, we worked with features of BERT's predicted vectors and degrees of contextuality, and derived an equation between the two. More specifically,  we showed that the differences in  BERT's logit scores  can be used as a proxy for the values of $\epsilon$. The former   is directly proportional to Euclidean distance,  and the latter is used to compute the entries of  the empirical tables.  We then trained a polynomial  regression model and computed  the correlations between degrees of (CbD and sheaf theoretic) contextuality and BERT features. This  provided further experimental support for our theoretical result, that  Euclidean distance offers the most predictive power for both CbD and sheaf theoretic contextuality. 

The percentage of contextual instances and  the degrees of correlation with contextuality were much higher in the similar noun subset of the dataset in comparison to the full dataset. Figure \ref{fig:simthreshold} plots the $R^2$ for different similarity thresholds for the full dataset. This plot shows that the number of contextual instances increases as we increase the similarity thresholds. These results match the ones we obtained previously, ~\cite{Lo2022,Lo2023}, where we had a  much smaller dataset,  constituting of only 11 pairs of nouns and 11,052 empirical models (vs  866,108 pairs of nouns and 51,966,480 empirical tables of the current paper).  The noun pairs of that dataset were  chosen to be highly similar, e.g. \ (cat, dog), (girl, boy) and (man, woman).  Working with highly similar noun pairs led to 350 sheaf-contextual (3.1$\%$) and 9,321 CbD-contextual (84$\%$) models. We conjecture that the higher percentages of contextuality among similar nouns   is because Euclidean and cosine distances are related to each other. 
Since the cosine similarity is related to the Euclidean distance via the equation $\| u - v \|^2 = \|u\|^2 + \|v\|^2 + \|u\| \|v\| \cos(u,v)$,  if the lengths of the vectors are roughly the same, then the cosine similarity (anti)-correlates with the Euclidean distance with to a high degree. In the similar noun subset the mean and standard deviation of the lengths of the noun embeddings are 1.19 and 0.14, and the said correlation is 0.60.

The fact that there are overwhelmingly more CbD-contextual models than sheaf-contextual models in our results is intriguing and raises many questions.
This discrepancy highlights fundamental differences in the criteria for contextuality between these two frameworks.
The contextual bounds in the ($\Delta$, $\SF$) space of empirical models are orthogonal to each other, as depicted in Figure \ref{fig:sfdeltaheatmap}.
This orthogonality suggests that the two frameworks are capturing different aspects of contextuality.
In the sheaf framework, when a signalling model is contextual, any ontological (i.e.\ hidden variable) explanation of the empirical data must be more signalling than the observed data.
In the CbD framework, whenever it is impossible to ``glue'' the empirical distributions together, the approach is to take a maximal coupling whose interpretation is less clear.
These differences underscore the complexity of contextuality in signalling data and exploring their significants in nautral language data and tasks is left for future work.



To conclude, we  demonstrated that  a variant of quantum contextuality can be observed in natural language data. Contextuality leads to quantum advantage.  It remains to show   whether quantum-like contextuality also leads to advantage, and if so what kind of advantage will it be and how can it be obtained. Finding answers to these questions is a  future direction.  In the meantime, one also needs to substantiate how this potential advantage can be used in improving methods that natural language tasks. 
Our linguistic schema is closely related to a well known coreference resolution task known as the Winograd Schema Challenge (WSC)~\cite{Levesque2012}. 
WSC was proposed as a benchmark for measuring machine intelligence.
The idea behind it is that solving the task requires common sense and access to external knowledge, which humans have but machines do not. 
In previous work, we showed that the measurement scenario of the original WSC is too simple to host contextuality~\cite{Lo2023GenWin}. 
The schema presented in this paper offers a suitable generalisation of it. 
In this paper, we showed how machines, i.e.\ the large language model BERT can be used to solve it. It remains to collect human judgements and compare their performances. 

An advantage of transformer-encoder models with bidirectional attentions such as BERT over the state of the art decoder models such as the GPTs is that its encoder architecture allows masked language modelling, which is a crucial tool for obtaining probability distributions for the instances of our linguistic schema.
It could be possible to use GPTs for our purpose with a carefully designed prompt, which is a future direction.

Coreference ambiguity has given rise to other historical challenges in Linguistics and Computational Linguistics. 
The most difficult cases arise when pronouns are used together with  quantifiers and indefinites. 
The term ``donkey anaphora'' is used to denote a family of such challenges, defying compositionality and posing many challenges to existing formal models of syntax and semantics. Donkey anaphora have been treated using sheaves \cite{AbramskySadr2014}. Collecting data for these examples using large language models and investigating whether they can host quantum contextuality is another future direction.  

\section{Supplementary Material}
The code for the experiments in this paper can be found in the following GitHub repository: \url{https://github.com/kinianlo/Contextuality-in-LLM}.

\begin{enumerate}
\item The code for adjective-noun phrases extraction from the Simple English Wikipedia can be found in the file \texttt{adjective\_noun\_phrases.py}.
\item The code for constructing instances of the PR-anaphora schema can be found in the file \texttt{construct\_pr\_anaphora.py}.
\item The code for calculating contextuality measures and features can be found in the file \texttt{contextuality\_measures.py}.
\item The code for evaluating the correlation between the features and contextuality can be found in the file \texttt{correlation\_analysis.py}.
\item Datasets used in the paper can be found in the \texttt{data} directory.
\end{enumerate}

\bibliographystyle{RS}
\bibliography{references.bib}
\end{document}